\RequirePackage{fix-cm}
\documentclass{svjour3}                     
\smartqed  
\usepackage{graphicx,psfrag,epsf}
\usepackage{url} 
\usepackage{hyperref}
\usepackage{amsmath,amssymb,amsfonts,bm,bbm} 
\usepackage{color}
\usepackage{enumitem}
\usepackage{mathtools}
\usepackage{algpseudocode}
\usepackage{algorithmicx, algorithm}
\usepackage{booktabs}
\usepackage{subcaption}

\usepackage[T1]{fontenc}
\usepackage[utf8]{inputenc}

\def\a{\alpha}
\def\b{\beta}
\def\d{\delta}
\def\e{\epsilon}
\def\g{\gamma}
\def\o{\omega}
\def\p{\phi}

\def\k{\kappa}
\def\lam{\lambda}
\def\L{\Lambda}

\def\P{\Phi}
\def\bssigma{\boldsymbol{\sigma}}
\def\bw{\mathbf w}
\def\bz{\mathbf z}

\def\bx{\mathbf x}

\def\bv{\mathbf v}

\def\bW{\mathbf W}
\def\bV{\mathbf V}

\def\bB{\mathbf B}
\def\R{\mathbb R}
\def\P{\mathbb P}
\def\E{\mathbb E}

\def\l{\left}
\def\r{\right}
\def\la{\l\langle}
\def\ra{\r\rangle}
\def\ll{\left\lVert}
\def\rl{\right\rVert}
\def\lv{\left\lvert}
\def\rv{\right\rvert}
\def\({\left(}
\def\){\right)}
\def\[{\left[}
\def\]{\right]}

\def\pt{\partial}
\def\nb{\nabla}
\def\cd{\cdot}
\def\ds{\displaystyle}

\def\qd{\quad}

\def\t{\tilde}
\def\h{\hat}

\def\M{M}

\def\hp{\hat{p}}
\def\pinf{p_\infty}

\def\hd{\h{\d}}

\def\a{\alpha}
\def\k{\kappa}
\def\b{\beta}
\def\d{\delta}
\def\e{\epsilon}
\def\g{\gamma}
\def\o{\omega}
\def\p{\phi}

\def\p{\rho}

\def\lam{\lambda}

\def\E{\mathbb E}

\def\R{\mathbb R}

\def\bw{{\bf{w}}}
\def\E{\mathbb{E}}
\def\bW{{\bf{W}}}
\def\hC{\h{C}}
\def\tL{\t{L}}
\def\M{\mathcal{M}}

\def\l{\left}
\def\r{\right}
\def\la{\left\langle}
\def\ra{\right\rangle}
\def\ll{\left\lVert}
\def\rl{\right\rVert}
\def\lv{\left\lvert}
\def\rv{\right\rvert}
\def\({\left(}
\def\){\right)}
\def\[{\left[}
\def\]{\right]}

\def\pt{\partial}
\def\nb{\nabla}
\def\cd{\cdot}
\def\ds{\displaystyle}

\def\qd{\quad}

\def\h{\hat}
\def\t{\tilde}

\def\kk{\h{\mu}}
\def\p{\psi}

\def\st{*}
\def\T{T}
\def\L{F}

\def\at{\frac{\a^2}{2}}

\def\M{\psi_\infty}
\def\o{C_L}
\usepackage{xr}
\begin{document}

\title{On Large Batch Training and Sharp  Minima: A Fokker-Planck Perspective
}

\titlerunning{On Large Batch Training and Sharp  Minima}        

\author{Xiaowu Dai         \and
        Yuhua Zhu 
}


\institute{Xiaowu Dai \at
              Department of Economics, University of California, Berkeley, CA, USA \\
              \email{xwdai@berkeley.edu}           
           \and
           Yuhua Zhu \at
             Department of Mathematics, Stanford University, CA, USA
}

\date{Received: date / Accepted: date}

\maketitle

\begin{abstract}
We study the statistical properties of the dynamic trajectory of stochastic gradient descent (SGD). 
We approximate the mini-batch SGD and the momentum SGD as stochastic differential equations (SDEs).
We exploit the continuous formulation of SDE and the theory of Fokker-Planck equations to develop new results on the escaping phenomenon and the relationship with large batch and sharp minima. 
In particular, we find that the stochastic process solution tends to converge to flatter minima regardless of the batch size in the asymptotic regime. However, the convergence rate  is rigorously proven to depend on the batch size. These results are validated empirically with various datasets and models.

\end{abstract}

\section{Introduction}
\label{sec:intro}

We consider the following empirical risk minimization problem in statistical machine learning:
\begin{equation*}
\min_{\bw\in\R^d}\frac{1}{N}\sum_{n=1}^NL_n(\bw),
\end{equation*}
where $\bw$ represents the model parameters, $L_n(\bw)$ denotes the loss due to the $n^{\text{th}}$ training sample, and $N$ is the size of the training set. 
Since the training set for many application domains such as image (He et al. \cite{he2016}) and speech recognition (Amodei et al. \cite{amodei2016deep}) is of large size, the stochastic gradient descent (SGD) and its variants have become standard approaches of training complex model including deep neural networks (Bottou et al. \cite{bottou2018}). The mini-batch SGD estimates the negative loss gradient based on a
small subset of training examples, which incurs the computational complexity per iteration independent of $N$:
\begin{equation}
\label{eqn:minibatchsgd}
\bw_{k+1}  = \bw_k - \frac{\g_k}{M_k}\sum_{n\in B_k} \nabla L_n(\bw_k),
\end{equation}
where $k\geq 0$,  $\gamma_k$ is the learning rate, and the mini-batch set $B_k$ consists of $M_k$ uniformly selected sample indices from $\{1,2,\ldots,N\}$.
A notable variant of mini-batch SGD is momentum SGD, which is a practical approach of speeding up the training (Nesterov \cite{nesterov2013}). 
For mini-batch SGD and its variant,  we use the term \emph{large batch training} to denote the use of a large mini-batch (Keskar et al. \cite{keskar}).

Recently, several works have discussed the geometry of SGD (Keskar et al. \cite{keskar}; Goyal et al. \cite{goyal2017}; Hoffer et al. \cite{hoffer}). 
Specifically, Keskar et al. \cite{keskar} find, based on empirical experiments, that the large batch training tends to converge to the sharp minima of the training function while the small batch training is more likely to escape the sharp minima. 
In this work, we study theoretically and empirically the dynamic of the convergence and escaping phenomenon relating to the batch size for mini-batch SGD and momentum SGD.

We approximate SGD using continuous stochastic differential equation (SDE) (Chaudhari et al., \cite{Chaudhari2017deep}; Mandt et al. \cite{mandt2017}; Li et al. \cite{li2017}).
 Assuming isotropic gradient noise, we derive new results on the dynamic trajectory of the Fokker-Planck solution. In particular, the derived convergence rate in terms of the batch size provides new insights into the escaping phenomenon for mini-batch SGD and momentum SGD. 
 Our main finding is that the stochastic process solution of SDE tends to converge to flatter minima regardless of the batch size in the asymptotic regime. However, the convergence rate depends on the batch size. 
 Motivated by partial differential equation theory, we define the sharpness in terms of the determinant of the Hessian, which provides a new perspective into the ongoing discussion on the definition of the sharpness  (e.g., Dinh et al. \cite{Dinh}). 
 We verify our theoretical results experimentally on different datasets and deep neural network models. 
 The proposed statistical view using tools from the Fokker-Planck equation can be used to analyze other stochastic algorithms for complex models.

The rest of the paper is organized as follows. We introduce the background in Section \ref{sec:minnibatchsgd}. We present our main result for mini-batch SGD in Section \ref{sec:mainresults}. 
We extend the result to momentum SGD in Section \ref{sec: model}.
We show numerical experiments in Section \ref{sec:simulation}. Related works are provided in Section \ref{sec:relatedwork}. We conclude the paper with discussions in Section \ref{sec:discussion}. Proofs are given in the Appendix.



\section{SDE Modeling for Large Batch Training}
\label{sec:minnibatchsgd}

The mini-batch SGD carries out the update at each step following (\ref{eqn:minibatchsgd}), which can be rewritten as
\begin{equation}
\label{minsgdre}
\bw_{k+1}  - \bw_k  = - \gamma_k\nabla L(\bw_k) + \frac{\gamma_k}{\sqrt{M_k}}\cdot\boldsymbol{\epsilon}_k,
\end{equation}
where $L(\bw) \equiv \E[L_n(\bw)]$ is the risk function and $\boldsymbol{\epsilon}_k = \frac{1}{\sqrt{M_k}}\sum_{n\in B_k}(\nabla L(\bw_k)-\nabla L_n(\bw_k))$ is a $d$-dimensional random vector.
Assume that the covariance matrix $\text{Var}[\nabla L_n(\bw)]\equiv\bssigma^2(\bw)$  is positive definite, which holds for typical loss functions including the squared loss.
By the dominated convergence theorem, $\boldsymbol{\epsilon}_k$ has mean $0$ and covariance $\bssigma^2(\bw_k)$  for any $k\geq 0$ (see, Appendix \ref{sec:meanvarep}).

For the large batch training, the distribution of $\boldsymbol{\epsilon}_k$ is well approximated by the normal distribution from the central limit theorem.
Consider the following stochastic differential equation (SDE) model:
\begin{equation}
\label{eqn:sde}
d\bW(t) = -\nabla L(\bW(t))dt - \sqrt{\frac{\gamma(t)}{M(t)}}\bssigma(\bW(t))d\bB(t), \ \ \bW(0) = \bw_0,
\end{equation}
where the Brownian motion $\bB(t)$ accounts for random fluctuations due to the use of mini-batches for gradient estimation in (\ref{minsgdre}). 
The Euler discretization of SDE (\ref{eqn:sde}) resembles the mini-batch SGD (\ref{minsgdre}), and the SDE solution approximates the mini-batch SGD in the weak sense (i.e., in distribution) under the finite-time setting $t\in[0,T]$ for any $T>0$; see, e.g., Li et al. \cite{li2017} and Mandt et al. \cite{mandt2017} .

\subsection{Escaping Phenomenon}
\label{sec:escapingpho}

Recently, Keskar et al. \cite{keskar}  note the escaping phenomenon of mini-batch SGD in training neural networks. Namely, the large batch training tends to converge to the sharp minima of the training function while
the small batch training is more likely to \emph{escape} the sharp minima. 
 A conceptual sketch of ``sharp" (and relatively, ``flat") minima are shown in Figure \ref{fig:conceptualflatsharp}, where  a mathematical definition of the sharpness is given in Section \ref{sec:finitetime}.
 Based on the numerical experiments, Keskar et al. \cite{keskar}  also find that a sharp minimum is correlated with a worse generalization, which, however, will not be studied in the current paper. 

\begin{figure}
    \centering
    \includegraphics[width=\textwidth]{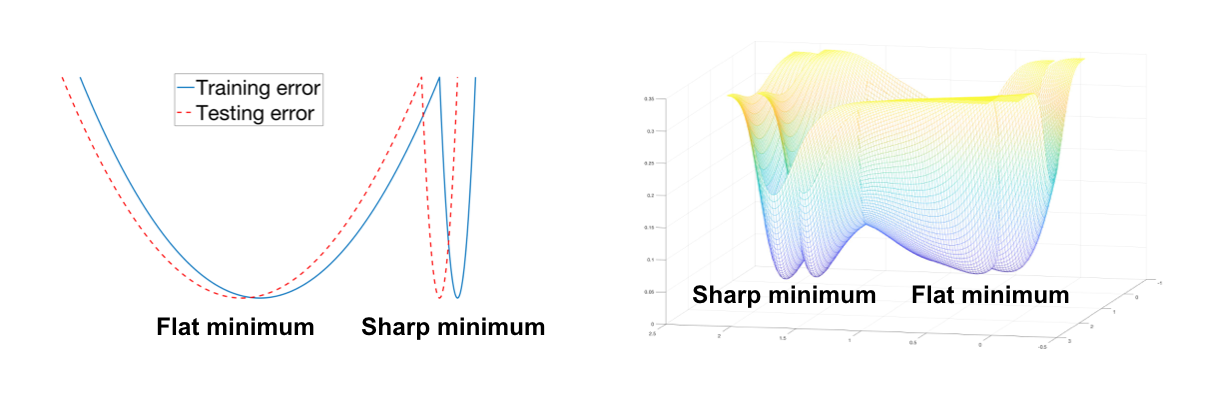}
    \caption{A  sketch of ``flat" and ``sharp" minima for one-dimensional case (left plot) and two-dimensional case (right plot). The vertical axis indicates the value of the loss function.}
    \label{fig:conceptualflatsharp}
\end{figure}

The escaping phenomenon is important for understanding the algorithm design for complex statistics and machine learning models. The phenomenon  has been validated in extensive numerical results; see, e.g., \cite{goyal2017} \cite{hoffer}. 
However, the theoretical support for the phenomenon is limited in the literature. The current paper fills some gaps in this important direction. Our approach is to use the SDE model (\ref{eqn:sde}) and study the escaping phenomenon for the stochastic process solution to the SDE model. 

\subsection{Fokker-Planck Equation}
\label{sec:fokkerplanck}
We allow the learning rate $\gamma_k$ and the batch size $M_k$ in (\ref{minsgdre}) to be varied along the step $k$, which is  consistent with the practice. 
As a result, the functions $\gamma(t)$ and $M(t)$ in (\ref{eqn:sde}) are  time-dependent.
Consider the isotropic gradient covariance: 
\begin{equation}
\label{eqn:isotconv}
\bssigma^2(\bw)= \beta(\bw)\cdot\mathbf{I},
\end{equation} 
where the scalar function $\beta(\bw)$ depends on $\bw$. Similar assumptions as (\ref{eqn:isotconv}) have been made in the stochastic algorithm literature, for example,  \cite{Chaudhari2017deep} \cite{Jastrzebski2017},  where $\beta(\bw) \equiv \beta$ is restricted to a constant. Since our interest lies in the escaping phenomenon and the relationship with the scale of  variance, the learning rate, and the batch size, we make the isotropic assumption (\ref{eqn:isotconv}) for simplicity and leave the anisotropic case for future study.
 
Denote by $p(\bw,t)$ the probability density function of the stochastic process solution $\bW(t)$.
We can characterize $p(\bw,t)$ in the following lemma, which is from the partial differential equations literature (e.g., Kolpas et al. \cite{kolpas2007coarse}).
\begin{lemma}
\label{lem:fokkerplank}
The  probability density function $p(\bw,t)$ satisfies the following  Fokker-Planck equation:
\begin{equation}
\label{eqn:probdensitypthetatw}
\begin{aligned}
 \partial_tp(\bw,t) = \nabla\cdot\left(\left[\nabla \left(L(\bw)+ \frac{\gamma(t)\beta(\bw)}{2M(t)}\right)\right]p(\bw,t) +  \frac{\gamma(t)\beta(\bw)}{2M(t)}\nabla p(\bw,t)\right),
\end{aligned}
\end{equation}
where $p(\bw,0)=\delta(\bw_0)$, and  $\delta(\cdot)$ denotes the Dirac's delta function.
\end{lemma}

We give a proof in Appendix \ref{sec:proofoflemfokkerplank}. 
Note that the drift term in (\ref{eqn:probdensitypthetatw}) $\nabla [L(\bw)+ \gamma(t)\beta(\bw)/2M(t)] \neq \nabla L(\bw)$, which implies that
the stochastic process solution $\bW(t)$ does not follow the mean drift direction $-\nabla L(\bw)$ as its update direction. A smaller batch size $M(t)$ corresponds to a drift term deviates further from the mean drift direction. 

\subsection{Kramer's Formula}
\label{sec:finitetime}

\begin{figure}
    \centering
    \includegraphics[width=0.5\textwidth]{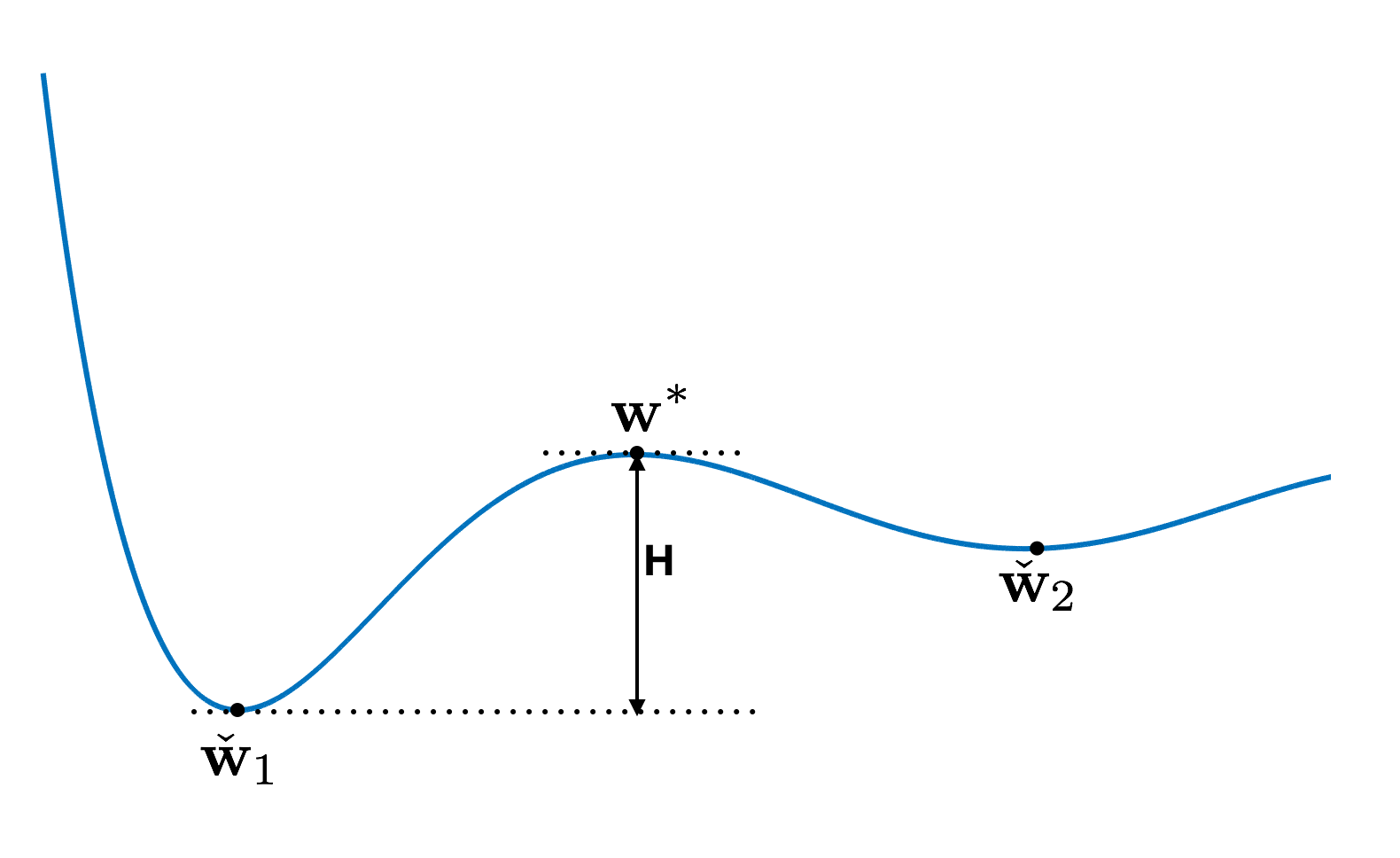}
    \caption{A  sketch of two local minimizer $\check{\bw}_1$ and $\check{\bw}_2$ of a risk function. The $\bw^*$ is the saddle point between $\check{\bw}_1$ and $\check{\bw}_2$.}
    \label{fig:finite_time_pic}
\end{figure}

Based on Fokker-Planck equation in Lemma \ref{lem:fokkerplank}, we are able to characterize the dynamics of the stochastic process solution in the finite-time regime. In particular, we have the escaping time of the stochastic process solution  from one local minimizer, denoted by $\check{\bw}_1$, to its   nearest local minimizer, denoted by $\check{\bw}_2$. Figure \ref{fig:finite_time_pic} gives an illustration, where $\bw^*$ is the saddle point between $\check{\bw}_1$ and $\check{\bw}_2$.  
There are possibly multiple saddle points between $\check{\bw}_1$ and $\check{\bw}_2$ in the multi-dimensional setting, where the  $\bw^*$ should be defined as the saddle point with the minimal height among all saddle points in the following sense. Denote by $\bw(t),0\leq t\leq 1,$ be any continuous path from $\check{\bw}_1$ to $\check{\bw}_2$, and $\widehat{\bw} = \arg{\text{inf}}_{\bw:\bw(0)=\check{\bw}_1, \bw(1) = \check{\bw}_2}\sup_{t\in[0,1]}L(\bw(t))$ the path with the minimal saddle point height among all continuous path. Then, $\bw^* \equiv \max_{t\in[0,1]}\widehat{\bw}(t)$.
It is known that the Hessian $\nb^2 L(\bw^*)$ has a single negative eigenvalue (e.g., Berglund \cite{berglund2013kramers}). 
Let $-\lam^*$ be the negative eigenvalue of $\nb^2 L(\bw^*)$  and $H(\bw^*,\check{\bw}_1) \equiv L(\bw^*) - L(\check{\bw}_1)$ be the relative height of $\bw^*$ to $\check{\bw}_1$. 
We have the following lemma characterizing the escaping time of the stochastic process solution from $\check{\bw}_1$ to $\check{\bw}_2$.

\begin{lemma}
\label{thm: exitingtime}
Let $\tau_{\check{\bw}_1 \to \check{\bw}_2}$ be the transition time for $\bW(t)$ from a closed ball of radius $\e>0$ centered at $\check{\bw}_1$ to a closed ball of radius $\e>0$ centered at $\check{\bw}_2$. Then
\begin{equation*}
\mathbb{E}[\tau_{\check{\bw}_1 \to \check{\bw}_2}] = \frac{2\pi}{\lam^*} \sqrt{\frac{|\Delta L(\bw^*)|}{|\Delta L(\check{\bf{w}}_1)|}}\exp\left(\frac{H(\bw^*,\check{\bw}_1) \cdot 2M(\check{\bw}_1)}{\gamma(\check{\bw}_1)\b(\check{\bw}_1)}\right) \left[1+O\(\sqrt{\e}\log\left(\frac{1}{\e}\right)\)\right], \label{eq: extingtime}
\end{equation*}
where $|\nb^2L(\cdot)|$ denotes the determinant of $\nb^2L(\cdot)$, $M(\check{\bw}_1)$ is the batch size at $\check{\bw}_1$, $\gamma(\check{\bw}_1)$ is the learning rate  at $\check{\bw}_1$, and $\beta(\cdot)$ is defined in (\ref{eqn:isotconv}).
\end{lemma}
Similarly, we have the transition time from $\check{\bw}_2$ to $\check{\bw}_1$ (i.e., $\tau_{\check{\bw}_2 \to \check{\bw}_1}$) with the only difference that the right side of the equation in Lemma \ref{thm: exitingtime} should be replaced by the geometry related to $\check{\bw}_2$.
This lemma is known in the diffusion process literature as the Eyring-Kramers formula; see, e.g.,  \cite{berglund2013kramers}, \cite{bovier2004metastability}, \cite{bovier2005metastability}.  Our observation is that the Eyring-Kramers formula can provide a quantitative description of  the escaping phenomenon in the \emph{finite-time} regime.
 In particular,  the time that  $\bW(t)$ escapes from one local minimum to its nearest local minimum
 depends on three factors. Namely, the diffusion factor $\gamma(\bw)\beta(\bw)/M(\bw)$, the potential barrier $H(\bw^*,\check{\bw}_1)$ that   $\bW(t)$ has to climb to escape $\check{\bw}_1$, and the determinants of the Hessians of the risk function at $\check{\bw}_1$ and $\bw^*$. This fact suggests the following definition of the sharpness. 
 \begin{definition}[Sharpness]
 \label{def:sharpness}
The sharpness of a minimizer is defined as the determinant of the Hessian of the risk function at the minimizer, i.e., $|\nb^2L(\cdot)|$. A larger $|\nb^2L(\cdot)|$ corresponds to a sharper minimizer.
 \end{definition} 

Lemma \ref{thm: exitingtime} shows that a larger batch size $M(\check{\bw}_1)$ at a local minimizer $\check{\bw}_1$ results in a longer time to escape from $\check{\bw}_1$. 
Hence, if $\check{\bw}_1$ corresponds to a sharp minimum with a large $|\Delta L(\check{\bf{w}}_1)|$, the exponential term 
\begin{equation}
\label{eqn:exph2mhatw}
\exp\left(\frac{H(\bw^*,\check{\bw}_1) \cdot 2M(\check{\bw}_1)}{\gamma(\check{\bw}_1)\b(\check{\bw}_1)}\right)
\end{equation} 
dominates the escaping time. 
Since a local minimizer of the training function lies in a closed ball of a local minimizer of the risk function, the stochastic process $\bW(t)$ of large batch training will be trapped at a sharp minimizer in the finite-time regime. This result shows  that \emph{large batch training tends to converge to sharp minima of the training function}.
On the other hand, if the batch size $M(\check{\bw}_1)$ decreases, the exponential term (\ref{eqn:exph2mhatw}) decreases, and the stochastic process  solution $\bW(t)$ will be  trapped at $\check{\bw}_1$ only when the determinant $|\nb^2L(\check{\bw}_1)|$ is small enough, as shown in Lemma \ref{thm: exitingtime}. In words, it explains the escaping phenomenon that \emph{small batch training tends to escape sharp minima and converge to flat minima}. 

The escaping phenomenon in the asymptotic regime is different from that in the finite-time regime. However,  the Eyring-Kramers formula fails when $t\to\infty$. We develop a new theory for the asymptotic regime in the following Section \ref{sec:mainresults}, and extend the result for momentum SGD-related SDE in Section \ref{sec: model}.

\section{Convergence Properties for Large Batch Training}
\label{sec:mainresults}
We study the  stochastic process solution $\bW(t)$ of the SDE  (\ref{eqn:sde}) in the asymptotic regime (i.e., $t\to\infty$). 
\subsection{Main Assumptions}
The main assumptions are outlined as follows.
\begin{itemize}
\item[](A.1) The risk function $L(\bw)$ is  confinement in the sense that 
\begin{equation*}
\lim_{\|\bw\|\to +\infty}L(\bw) = +\infty,\quad \int e^{-L(\bw)}d\bw < +\infty.
\end{equation*}
\item[] (A.2)  Denote by $\text{Tr}(\nb^2L)$ the trace of the Hessian of $L$. Assume 
\begin{equation*}
\begin{aligned}
& \lim_{\|\bw\|\to +\infty}  \left\{\frac{1}{2}\|\nb L(\bw)\|^2 - \text{Tr}(\nb^2L(\bw))\right\} = + \infty, \\
& \lim_{\|\bw\|\to +\infty}  \left\{\text{Tr}(\nb^2L(\bw))/\|\nb L(\bw)\|^2\right\} = 0.
\end{aligned}
\end{equation*} 
\item[](A.3)  There exists a constant $M_\bw$, such that 
\begin{equation*}
\lv e^{- L(\bw)}\(\ll \nb L(\bw) \rl^2 - \text{Tr}(\nb^2L(\bw))\) \rv \leq M_\bw.
\end{equation*}
\end{itemize}
Assumptions (A.1)--(A.3) is common in the diffusion process literature, see, e.g., Pavliotis \cite{pavliotis2014stochastic}.
We show in Appendix \ref{sec:disonassump} that (A.1)--(A.3)  hold for typical loss functions, including the regularized mean cross entropy and the squared loss.
In particular, Assumption (A.1) ensures that the  Gibbs  density function $  e^{-L(\bw)}$ is well defined. Assumption (A.2) guarantees the measure $\mu(\bw) =   \int e^{-L(\bw)}d\bw$ satisfying the Poincar\'e inequality (see, \cite{pavliotis2014stochastic}):
\begin{equation}
\label{eqn:poincare}
\begin{aligned}
  \int \ll \nb f(\bw) \rl^2d\mu(\bw)  \geq C_P\int \(f(\bw) - \int f(\bw) d\mu(\bw)\)^2 d\mu(\bw)
  \end{aligned}
\end{equation}
with some $C_P>0$, where $f$ is any integrable function satisfying $\int f^2(\bw)d\bw<\infty$.
\begin{lemma}
\label{lem:stationbeta}
Under Assumption (A.1) and  $\beta(\bw)\equiv\beta$, the Fokker-Planck equation  (\ref{eqn:probdensitypthetatw}) has a stationary solution  in the asymptotic regime (i.e., $t\to\infty$):
\begin{equation*}
\label{eqn:defofpinf}
p_\infty(\bw) = \kappa e^{- \frac{2M(\infty)L(\bw)}{\gamma(\infty)\b}}, 
\end{equation*}
where  $\kappa$ is a normalization constant such that $\int p_\infty(\bw)d\bw=1$, and the limiting batch size and learning rate are defined as $M(\infty)\equiv\lim_{t\to\infty}M(t)$ and  $\gamma(\infty)\equiv\lim_{t\to\infty}\gamma(t)$, respectively.
\end{lemma}
A derivation of Lemma \ref{lem:stationbeta} is provided in Appendix \ref{sec:proflemstationbeta}.
We remark that for a general $\beta(\bw)$ which depends on $\bw$, the existence and an explicit form of  stationary solution to  the Fokker-Planck equation (\ref{eqn:probdensitypthetatw}) remains an open question.
We focus on  $\beta(\bw)\equiv\beta$ in this section.


\subsection{Escaping Phenomenon in the Asymptotic Regime}

Related works on  the analysis of stochastic algorithms have studied the stationary solution $p_\infty(\bw)$; see, e.g., Jastrzebski et al. \cite{Jastrzebski2017}. 
However, it is unclear whether the density function $p(\bw,t)$ converges to the stationary solution $p_{\infty}(\bw)$, not to mention the convergence rate.
Theorem \ref{thm:sigm4.1proof} gives an affirmative answer to this problem, and it also provides new insights into the escaping phenomenon and the relationship with large batch and sharp minima.
\begin{theorem}
\label{thm:sigm4.1proof}
Under Assumptions (A.1)--(A.3), the density function $p(\bw,t)$  of $\bW(t)$    converges to the stationary solution $\pinf(\bw)$.
Moreover, there exists $T>0$ such that for any $t>T$,  
\begin{equation*}
 \ll \frac{p(\bw,t) - p_\infty(\bw)}{\sqrt{p_\infty(\bw)}} \rl_{L^2(\R^d)}^2 \leq C(t,T)e^{-\frac{C_P\cdot (t-T)\cdot \gamma(\infty)\beta}{2M(\infty)}},
\end{equation*}
where the constant $C_P$ is  defined in (\ref{eqn:poincare}), and the function $C(t,T)$ is given by
\begin{equation*}
C(t,T) \equiv \frac{C_P\cdot (t-T)\cdot\gamma(\infty)\beta}{2M(\infty)} +  \ll\frac{{p(\bw, T) - p_\infty(\bw)}}{\sqrt{p_\infty(\bw)}} \rl_{L^2(\R^d)}^2.
\end{equation*} 
\end{theorem}
Theorem \ref{thm:sigm4.1proof} is new in the literature, and  its proof is given in Appendix \ref{sec:proofofthm:sigm4.1proof}. We also give a quantification of the constant $T$ in Appendix \ref{sec:quantificT}.
We make three remarks for Theorem \ref{thm:sigm4.1proof}.
First, the theorem verifies that $p(\bw,t)$ converges to the stationary solution $p_\infty(\bw)$ with an exponential convergence rate regardless of the initial value. This result provides  theoretical support for related works that analyze the density function $p(\bw,t)$ based on analysis of the stationary distribution $p_\infty(\bw)$, for example, Jastrzebski et al. \cite{Jastrzebski2017}.
Second, large batch training with increasing batch size converges exponentially slower.  
Finally, there exists a tradeoff in choosing the batch size and learning rate, since the convergence rate $\exp(-C_P\cdot (t-T)\cdot\gamma(\infty)\beta/2M(\infty))$ depends on the batch size $M$ and the learning rate $\gamma$.

From Theorem \ref{thm:sigm4.1proof}, we can also characterize the limiting behavior of $\bW(t)$ in the asymptotic regime when $t\to\infty$. 
\begin{theorem}
\label{thm:probmainresult4111}
Let $\check{\bw}$ be a  local minimizer. Then,
\begin{equation*}
\begin{aligned}
&\lim_{\e\to0}\P(|\bW(\infty) - \check{\bw}|\leq \epsilon)  \\
&=\frac{\kappa e^{-4M(\infty) L(\check{\bw})/[\gamma(\infty)\beta]}}{[2M(\infty)/\gamma(\infty)\beta]^{d/2}|\nb^2L(\check{\bw})|}\lim_{\e\to0}\l[e^{\frac{2M(\infty)\epsilon^2}{\gamma(\infty)\beta}}\prod_{j=1}^d\sqrt{1-\exp\left(-\frac{2M(\infty)\e^2\lam_j}{\pi\gamma(\infty)\beta}\right)}\r],
\end{aligned}
\end{equation*}
where  $\bw\in\R^d$, $\lambda_j$'s are eigenvalues of  the Hessian  $\nb^2L(\check{\bw})$, and
$|\nb^2L(\check{\bw})|$ is the determinant of  $\nb^2L(\check{\bw})$. The constants $\kappa$ and $\eta(\infty)$ are defined  in Lemma \ref{lem:stationbeta}.
\end{theorem}
The proof of Theorem \ref{thm:probmainresult4111} is given in Appendix \ref{eqn:proofofthfinal}.
To better appreciate Theorem \ref{thm:probmainresult4111}, we consider  two local minimizers 
$\check{\bw}_1$ and $\check{\bw}_2$ which have the same value of $L(\check{\bw}_1)=L(\check{\bw}_2)$.
Theorem \ref{thm:probmainresult4111} implies that
\begin{equation}
\label{eqn:ratioofprob}
\begin{aligned}
&\lim_{\e\to0}\frac{\P(|\bW(\infty) - \check{\bw}_1|\leq \epsilon) }{\P(|\bW(\infty) - \check{\bw}_2|\leq \epsilon)} = \sqrt{\frac{\lv\nb^2L(\check{\bw}_2)\rv}{\lv\nb^2L(\check{\bw}_1)\rv}},
\end{aligned}
\end{equation}
where the derivation is given in Appendix \ref{eqn:proofofeqnratio}.
Then, Equation (\ref{eqn:ratioofprob})  suggests that in the asymptotic regime (i.e., $t\to\infty$), the probability of the stochastic process solution $\bW(t)$ converging to a minimum with small determinant $|\nb^2L(\cdot)|$ is  larger than that of   converging to a minimum with  large determinant $|\nb^2L(\cdot)|$. In words, by Definition \ref{def:sharpness}, $\bW(t)$ is more likely to converge to flatter minima.
Moreover, the ratio in (\ref{eqn:ratioofprob}) does not depend on the batch size or learning rate, and only on the determinant of the Hessian at the minimum.

Theorems \ref{thm:sigm4.1proof} and \ref{thm:probmainresult4111} provide new insights into the escaping phenomenon in Section \ref{sec:escapingpho}. Namely, the stochastic process solution $\bW(t)$
 tends to converge to flatter minima regardless of the batch size $M$  in the asymptotic regime $t\to\infty$. However, the \emph{convergence rate} depends on the batch size. We provides experiments in Section \ref{sec:simulation} to corroborate these findings for mini-batch SGD with various datasets and neural network models.


\section{SDE Modeling for Momentum SGD}
\label{sec: model}
Momentum SGD (MSGD) is an effective approach of speeding up the mini-batch SGD; see, e.g., Qian \cite{qian1999momentum}, Nesterov \cite{nesterov2013}, Sutskever et al. \cite{Sutskever2013}.
Instead of updating $\bw_k$ directly in (\ref{eqn:minibatchsgd}), MSGD  adopts the following coupled updates:
\begin{equation*}
\label{eq: MSGD}
\begin{aligned}
	&\bz_{k+1} = \xi \cdot \bz_k - \frac{\g_k}{M_k}\sum_{n\in B_k}\nb L_n(\bw_k),\\
	&\bw_{k+1} = \bw_k + \bz_{k+1}.
\end{aligned}
\end{equation*}
where $\xi$ is the momentum parameter taking values in the range $0< \xi< 1$.  In this section, we focus on the constant learning rate and batch size: $\gamma_k\equiv \gamma, M_k\equiv M$, and leave the time-dependent case for future study. 
Let $\bv_k = \bz_k/\sqrt{\g}$.
When the step size is small, $(\bv_k, \bw_k)$ can be approximated by the SDE (see, e.g., Li et al. \cite{li2017}, An et al. \cite{an2019stochastic}), 
\begin{equation*}
\label{eq: msde}
\left\{
\begin{aligned}
&d\bV(t) = -\nb L(\bW(t)) dt -\frac{1 - \xi}{\sqrt{\g}} \bV(t) dt + \frac{\g^{1/4}}{\sqrt{M}}\sqrt{\b(W(t))} d\bB(t),\\
	&d \bW(t) = \bW(t)dt.	
\end{aligned}
\right.
\end{equation*}
where $\b(\bw)$ is the scale of the covariance function defined in \ref{eqn:isotconv}. The SDE modeling gives $\bV({k\sqrt{\g}}) \approx \bv_k$, $ \bW({k\sqrt{\g}}) \approx \bw_k $, which is shown in Appendix \ref{sec: proof of msde}.  

\subsection{Vlasov-Fokker-Planck Equation}
Denote by $\p(\bw,\bv,t)$ the joint probability density function of $(\bW(t), \bV(t))$. 
We have the following characterization of $\p(\bw,\bv,t)$ from the partial differential equations literature (e.g., Pavliotis \cite{pavliotis2014stochastic}), and also show the corresponding stationary solution.
\begin{lemma}
\label{lemma: pdf msgd}
The probability density function $\p(\bw,\bv,t)$ satisfies the following Vlasov-Fokker-Planck equation:
\begin{equation}
\label{eq: origin pde}
\begin{aligned}
    &\pt_t\p(\bw,\bv,t) + \bv \cdot\nb_\bw\p(\bw,\bv,t) -\nb L(\bw)\cdot\nb_\bv \p(\bw,\bv,t) \\
  &\quad\quad \quad\quad\quad  =  \nb_\bv \cdot\( \frac{1 - \xi}{\sqrt{\g}} \bv \p(\bw,\bv,t) + \frac{\sqrt{\gamma}\beta}{2M} \nb_\bv \p(\bw,\bv,t)\).
\end{aligned}
\end{equation}
Moreover, under  Assumption (A.1) and  $\b(\bw) \equiv \b$, the  equation (\ref{eq: origin pde}) has a stationary solution in the asymptotic regime (i.e. $t\to\infty$):
\begin{equation*}
   \psi_\infty(\bw,\bv ) = \k' e^{- \frac{2M}{\g\b}(1-\xi)\l(L(\bw) + \frac{\lv \bv \rv^2}2\r)},
\end{equation*}
where $\k'$ is a normalization constant such that $\int \psi_\infty d\bw d\bv  = 1$.
\end{lemma}

We give a proof in Appendix \ref{sec: proof of lemmsgd}. 
By integrating $\psi_\infty(\bw,\bv)$ over $\bv$, we obtain that $\int \psi_\infty(\bw,\bv) d\bv = \k e^{- \frac{2M}{\g\b}(1-\xi)L(\bw)}$, which is similar to the stationary solution in Lemma \ref{lem:stationbeta} and implies the equation (\ref{eqn:ratioofprob}) for MSGD. Hence, the stochastic process $\bW(t)$ for MSGD-related SDE tends to converge to flatter minima regardless of the batch size in the asymptotic regime $t\to\infty$.
However, we show in Section \ref{sec: main results} that the convergence rate depends on the batch size.

\subsection{Escaping Phenomenon of MSGD-Related SDE}
\label{sec: main results}

In this section, we require an additional assumption. 
\begin{itemize}
\item[](A.4) There exists a constant $\o$ such that the absolute values of eigenvalues of the matrix $\{\|(\nb^2\tL)_{ij}\|_{\infty}\}_{1\leq i,j\leq d}$ are bounded by a constant $b>0$, where $\tL(\bw) = L(\bw) - \frac12\o^2 \|\bw\|^2$ and $\{\|(\nb^2\tL)_{ij}\|_{\infty}\}_{1\leq i,j\leq d}$ consists of the $(i,j)$th entry  $\|(\nb^2\tL)_{ij}\|_{\infty} \equiv \sup_\bw|(\nb^2\tL)_{ij}|$.
\end{itemize}
We prove in Appendix \ref{sec:discassumpa4} that Assumption (A.4)  holds for typical loss functions including the regularized mean cross entropy and the squared loss. 
\begin{theorem}
\label{thm: main thm}
Under Assumption (A.1)--(A.4), the density function $\p(\bw,\bv,t)$ of $(\bW(t),\bV(t))$ converges to the stationary solution $\psi_\infty(\bw,\bv)$. Moreover, there exists $T>0$ such that for any $t>T$, 
\begin{equation*}
  \begin{aligned}
    &\ll \frac{\p(\bw,\bv,t) - \psi_\infty(\bw,\bv)}{\sqrt{\psi_\infty(\bw,\bv)}} \rl_{L^2(\R^{2d})}^2 \leq \frac{\g\b}{2M\min\{ C_P, d\}(1-\xi)\lam_{\min}} e^{-2(\mu - \kk)t} H(0).
  \end{aligned}
\end{equation*}
The parameters are specified as follows. First,  $C_P$ is the Poincar\'e constant defined in (\ref{eqn:poincare}). 
Define 
\begin{equation*}
h(\bw,\bv,t) \equiv \frac{\p(\bw,\bv,t) - \psi_\infty(\bw,\bv)}{\psi_\infty(\bw,\bv)}\quad \text{ and matrix }\quad P \equiv\l[ \begin{aligned} &I_d & \hC I_d\\
&\hC I_d &C I_d\end{aligned}\r],
\end{equation*}
where the constants $C$ and  $\hC$ together of the decay rate $\mu$ are determined by
\begin{equation*}
    \l\{ \begin{aligned}
      &\text{if } \frac{1 - \xi}{\sqrt{\g}} < 2\o: \mu \equiv  \frac{1 - \xi}{\sqrt{\g}}, \ C \equiv \o^2,\ \hC \equiv  \frac{1 - \xi}{2\sqrt{\g}};\\
      &\text{if } \frac{1 - \xi}{\sqrt{\g}} \geq 2\o: \mu \equiv  \frac{1 - \xi}{\sqrt{\g}} - \sqrt{ \frac{(1 - \xi)^2}{\g} - 4\o^2},  \ C \equiv  \frac{(1 - \xi)^2}{2\g}-\o^2, \ \hC \equiv  \frac{1 - \xi}{2\sqrt{\g}}.
    \end{aligned}
    \r.
  \end{equation*} 
Next, let $\lam_{\min}$ be the smallest eigenvalue of  the matrix $P$.  Finally, let
$$\ds H(0) = \int[\nb_\bw h(\bw,\bv,0), \nb_\bv h(\bw,\bv,0)]^\top P [\nb_\bw h(\bw,\bv,0), \nb_\bv h(\bw,\bv,0)]\psi_\infty d\bw d\bv,$$
and  $\ds \kk =  \frac{(1+\sqrt2)b}{2\lam_{\min}} $, where $b$ the upper bound defined in  Assumption (A.4),
\end{theorem}

From Theorem \ref{thm: main thm}, it is clear that the large batch training (i.e., as $M$ increases) has a slower convergence as compared with the small batch training. 
Proof of this theorem is given in Appendix \ref{sec: proof of main thm}. Theorem \ref{thm: main thm} is new in the literature, and it builds on the result for quadratic function $L(\bw) = \frac{1}{2}\o^2 \ll \bw \rl^2$ in the literature (e.g., Pavliotis \cite{pavliotis2014stochastic}).  Theorem \ref{thm: main thm} is applicable to general loss functions, including  the regularized mean cross entropy and the squared loss.

The main difficulty in the proof is that equation (\ref{eq: origin pde}) is a degenerate diffusion PDE in the sense that it only has the diffusion on the $\bv $ direction without the diffusion on the $\bw$ direction; see, an overview on the diffusion PDE in  Evans \cite{evans2010}. 
We use the tools from \emph{hypocoercivity} (see, Vallani \cite{vallani2009}), which links a 
degenerate diffusion operator and a conservative operator.
The key idea in the proof is to construct a Lyapunov functional  $H(t)$ (\cite{vallani2009}):
\begin{equation*}
  H(t) =\ll \nb_\bw h \rl^2_\st + C\ll \nb_\bv h \rl^2_\st + 2\h{C}\la\nb_\bw h, \nb_\bv h \ra_\st,
\end{equation*}
where $\la h, g\ra_\st \equiv \int  hg  \psi_\infty d\bw d\bv$ and $\ll h \rl_\st$ is the corresponding norm. The above equation can be equivalently written as,
\begin{equation}
  \label{def: P}
  H(t) = \int [ \nb_\bw h,   \nb_\bv h ]^\top P [ \nb_\bw h,   \nb_\bv h ] \psi_\infty d\bw d\bv ,\qd \text{with }P =  \left[ 
    \begin{aligned} 
    & I_d  &\h{C}I_d\\
    &\h{C}I_d &CI_d
    \end{aligned}\right].
\end{equation}
where $C, \h{C}$ are constants to be determined.  Note that
\begin{equation*}
  \frac{d}{dt}H(t) =\frac{d}{dt}\(\ll \nb_\bw h \rl^2_\st + C\ll \nb_\bv h \rl^2_\st\) +
  2\h{C}\frac{d}{dt}\la\nb_\bw h, \nb_\bv h \ra_\st,
\end{equation*}
which implies following inequality with some constant $\t{C}$,
\begin{equation*}
d_t H(t) + \t{C} H(t) \leq 0,
\end{equation*}
and the exponential decay of $H(t)$. Finally, the relationship between $H(t)$ and $h(t)$ in (\ref{def: P}) leads to the exponential decay for $\ll h(t) \rl_*^2$ as required for Theorem \ref{thm: main thm}.

\section{Numerical Experiments}
\label{sec:simulation}

We perform experiments using various datasets and deep learning models to corroborate 
 theoretical findings in Sections \ref{sec:minnibatchsgd}--\ref{sec: model}.

\subsection{Escaping Phenomenon for Mini-Batch SGD}
\label{sec:escapingsim}

We consider three different neural network models: a four-layer multilayer perception (MLP) with ReLU  activation function and batch normalization \cite{Ioffe2015batch}, a shallow convolutional network N1, and a deep convolutional network N2. The N1 network is  a modified AlexNet configuration (Krizhevsky et al., \cite{Krizhevsky2012imagenet}), and the N2 network is a modified  VGG configuration (Simonyan and Zisserman \cite{Simonyan2015}). 
 We test and train the MLP with the MNIST dataset (LeCun et al. \cite{LeCun1998}), and N1 and N2 with the CIFAR-10 dataset, using the mean cross entropy as  the loss function. Details on the networks and dataset are given in Appendix \ref{sec:netdataexp}.
 We study the escaping phenomenon of the mini-batch SGD with four pairs of learning rate and batch size: $(\gamma,M)$ $= (0.1,64),$ $(0.1,128),$ $(0.2,256),$ $(0.2,512)$.  
A total of $100$ epochs for each $(\gamma,M)$ are trained, where the training loss stops decreasing.
We repeat each experiment 100 times and average the results in Figure \ref{fig:f3}.
 Due to the high computational cost for computing the determinant of Hessian, we use the 
Frobenius norm of Hessian as a substitute, which is similar to Wu et al.  \cite{wu2017}.  A smaller $y$-value in  Figure \ref{fig:f3} indexes a flatter minimum.  
 The $x$-axis denotes the number of steps, which equals  $\text{the number of epoch }*N/M$, where $N$ is the training sample size, and $M$ is the batch size.

\begin{figure}[h!]
    \centering
    \includegraphics[width=\textwidth]{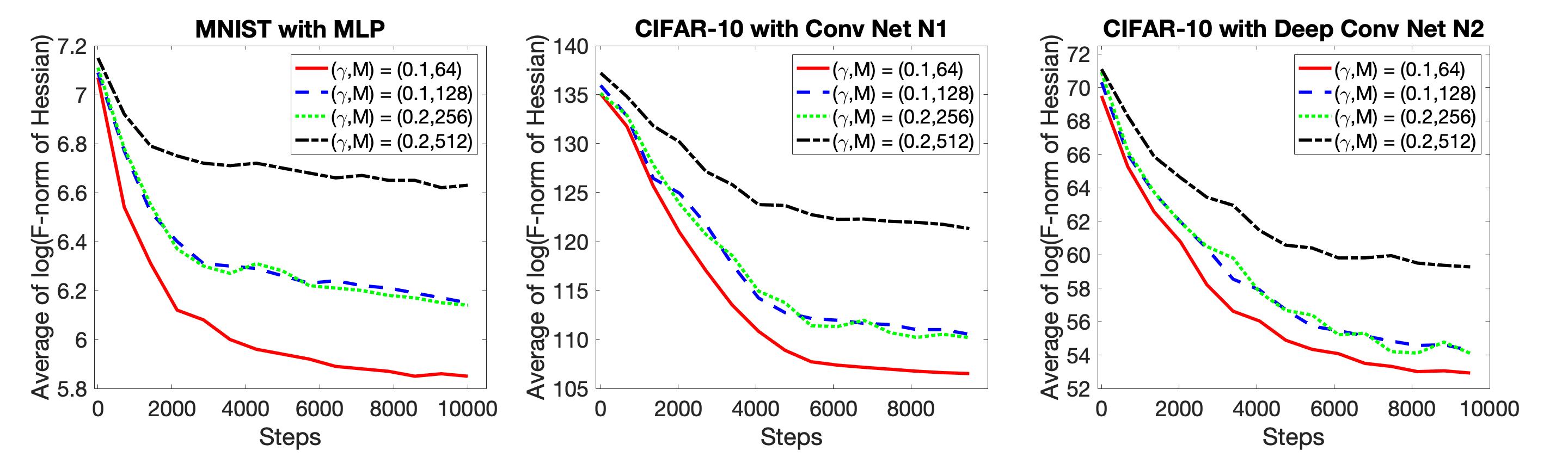}
    \caption{Log of Frobenius norm of Hessian as a function of steps. The left plot is 4-layer batch-normalized MLPs with  MNIST dataset. The middle plot is convolutional network N1 with CIFAR-10 dataset. The right plot is deep convolutional network N2 with CIFAR-10 dataset.
Four $(\gamma,M)$ pairs  are studied: $(0.1, 64),$ $(0.1,128),$  $(0.2,256)$, and $(0.2,512)$, which are denoted in 
 red, blue, green, and black, respectively. The plots show the averaged results of 100 experiments for each of the four $(\gamma,M)$ pairs. }
    \label{fig:f3}
\end{figure}

Figure \ref{fig:f3} shows  that under the same learning rate, the large batch training converges to sharper minima, for example, comparing the red solid curves with the blue dashed curves for all three plots, which agrees with Lemma \ref{thm: exitingtime}. The mini-batch SGD with the same $\gamma/M$ ratio follows a similar dynamic trajectory in terms of sharpness, which is consistent with the result of the SDE modeling in Lemma  \ref{lem:stationbeta}. 

In the asymptotic regime, Theorem \ref{thm:sigm4.1proof} shows that the large batch training converges to a flat minimum slower as compared with the small batch training. This is clear from Figure \ref{fig:f3}. For example, the black dash-dot curve in the right plot takes $10,000$ steps to converge at a minimum of  $y=58$, while the green dotted curve only takes $4,000$ steps to achieve it. 
On the other hand, for any batch size, SGD is more likely to saturate with a flatter minimum. For example, in the average case, the black curve in the right plot explores minima with $y$-values range from $58$ to $71$ while it ends up with a minimum of $y=58$, which corroborates Theorem \ref{thm:probmainresult4111}.




\subsection{Escaping Phenomenon for Momentum SGD}
\label{sec:escapingsimmsgd}

We empirically study the escaping phenomenon for momentum SGD. We use the three neural network models as Section \ref{sec:escapingsim}: MLP, convolutional network N1, and deep convolutional network N2, which are trained with MNIST, CIFAR-10, and CIFAR-10, respectively. We consider four pairs of momentum parameter and batch size: $(\xi, M) = (0.9,64),(0.9,128),(0.99,64),(0.99,128)$, while the learning rate is $\gamma=0.1$.
A total of $100$ epochs for each $(\xi,M)$ are trained, where the training loss stops decreasing near the ending of the training.
We repeat each experiment 100 times and average the results in Figure \ref{fig:f4}.

\begin{figure}[h!]
    \centering
    \includegraphics[width=\textwidth]{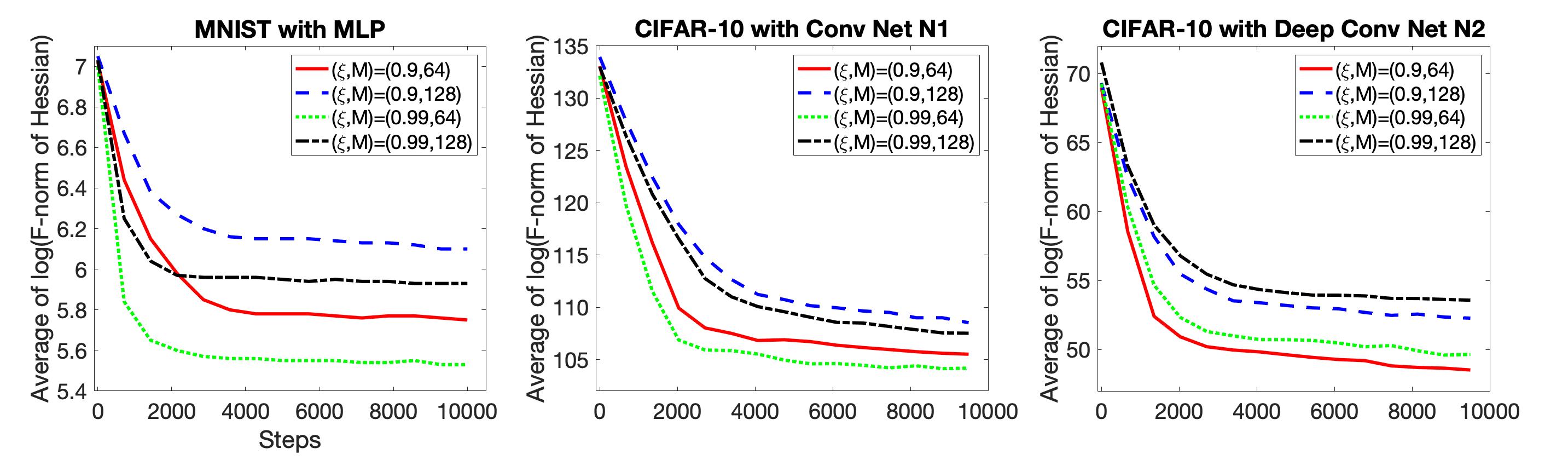}
    \caption{Log of Frobenius norm of Hessian as a function of steps. The left plot is 4-layer batch-normalized MLPs with  MNIST dataset. The middle plot is convolutional network N1 with CIFAR-10 dataset. The right plot is deep convolutional network N2 with CIFAR-10 dataset.
Four $(\xi, M)$ pairs  are studied: $(0.9,64),$ $(0.9,128),$  $(0.99,64)$, and $(0.99,128)$, which are denoted in 
 red, blue, green, and black, respectively. The plots show the averaged results of 100 experiments for each of the four $(\xi,M)$ pairs. }
    \label{fig:f4}
\end{figure}

Figure \ref{fig:f4} shows that under the same momentum parameter, the large batch training converges to sharp minima. 
In the asymptotic regime, Theorem \ref{thm: main thm} shows that the large batch training converges to a flat minimum slower compared to the small batch training, which is clear from Figure \ref{fig:f4}. For example, the blue dashed curve in the middle plot takes $8,000$ steps to converge at a minimum of $y=110$, while the red solid curve only takes $2,000$ steps to achieve it. This phenomenon is robust to the momentum parameter (e.g., $\xi=0.9$ or $0.99$). 
On the other hand, Theorem \ref{thm: main thm} suggests there is no monotonic rule for tuning the momentum parameter $\xi$ since both  $\lambda_{\min}$ and $\mu$ depend on $\xi$. We observe a similar pattern in Figure \ref{fig:f4}. While $\xi=0.99$ leads the momentum SGD to converge to flatter minima for MLP and N1 networks, $\xi=0.99$ ends up with sharper minima for N2.

\section{Related Work}
\label{sec:relatedwork}

Our work continues the line of research on the geometry of SGD, see, for example, Bottou et al. \cite{bottou2018} for a comprehensive review. In particular, our interest lies in the role of large batch size and the
sharpness of minima found in terms of generalization; see, e.g., \cite{keskar} \cite{goyal2017} \cite{hoffer}. 
Keskar et al. \cite{keskar}  find, based on empirical experiments, that the large batch training tends to converge to a sharp minimum. Goyal et al. \cite{goyal2017} and Hoffer et al. \cite{hoffer}  observed through experiments that training for more epochs and scaling up the learning rate give good generalization when using large batch size. 
This paper is complementary to the existing works in this direction. 
Motivated by partial differential equation theory, we define the sharpness in terms of the determinant of the Hessian, which provides a new perspective into the discussion on the definition of the sharpness (e.g., Dinh et al. \cite{Dinh}).  We explain theoretically and empirically the dynamic of the convergence and escaping phenomenon relating to the batch size for mini-batch SGD and momentum SGD. 

Several authors have developed the relationship between SGD and sampling a posterior distribution via stochastic Langevin methods; see, e.g., Chaudhari et al., \cite{Chaudhari2017deep}, Mandt et al. \cite{mandt2017}. In particular,  Mandt et al. \cite{mandt2017}  study SGD using an approximate Bayesian inference method in a locally convex setting. 
The modeling of SGD as a  continuous time stochastic process can also be achieved using SDE; see, e.g., \cite{li2017} \cite{Smith2018} \cite{chaudhari2017stochastic}. In particular,
Li et al. \cite{li2017} rigorously derive an approximation error of SDE solution to SGD in the finite-time regime. 
Smith and Le \cite{Smith2018} use Bayesian principles to relate the generalization error with the batch size. 
Chaudhari and Soatto  \cite{chaudhari2017stochastic} discuss the stationary non-equilibrium solution for the stochastic differential equation, where they allow the  gradient noise to be non-isotropic, but require additional conditions to
enforce the stationary distribution to be path-independent. 
Instead, we strictly focus on the convergence rate of the SDE solution to the stationary distribution with isotropic noise. This approach allows us to explore the dynamics of the convergence relating to the batch size and sharp minima, which results are verified empirically with various datasets and deep neural network models.

We discuss the Fokker-Planck equation and its variant, which modelings have appeared in the machine learning literature. 
Heskes and Kappen \cite{heskes1993} derive a Gibbs distribution in the online setting. Jastrzebski et al. \cite{Jastrzebski2017} discuss how the width and height of minima correlate with the learning rate to batch size ratio, but they focus on the stationary equilibrium distribution. 
Our result also show that the ratio of learning rate to batch size is correlated with sharpness of minima (e.g., Lemma \ref{lem:stationbeta}) in the stationary solution. 
In contrast to other work, we derive new results on the dynamic trajectory of the Fokker-Planck solution including the convergence rate in terms of the batch size, which provides new insights into the escaping phenomenon for mini-batch SGD and momentum SGD.

\section{Conclusion}
\label{sec:discussion}

We  study the convergence rate of the SDE solution to the stationary distribution, which is new in the literature and  allows us to explore the dynamics of the escaping phenomenon and the relationship with the batch size and sharp minima. The perspective from the Fokker-Planck equation and its variant provide novel insights into the escaping phenomenon for mini-batch SGD and momentum SGD. Namely, the stochastic process solution  tends to converge to flatter minima regardless of the batch size in the asymptotic regime. However, the convergence rate depends on the batch size. These results are validated theoretically and empirically with various datasets and deep neural network models.

We made the isotropic assumption on the covariance of the gradients, which is to derive a closed form for the convergence rate of the SDE solution to the stationary distribution. 
It is of interest to study whether the practical techniques such as batch normalization would give a covariance  of the gradients close to the isotropy. We also leave the study of extending this paper to anisotropic covariance structure for future work. 
Finally, the derived asymptotic dynamic reflects the transition dynamics of the SDE,  which is an idealization of SGD. 
For the asymptotic regime to directly represent the SGD escape dynamics, one requires the additional uniform-in-time approximation of SGD by SDE, which remains an open question for non-convex loss functions.


\appendix

\section{Proofs for Section  \ref{sec:minnibatchsgd}}
\label{sec:proofofexpvarlnbw}

\subsection{Mean and Variance for Random Error Vector}
\label{sec:meanvarep}
By the mean value theorem with some $\tau(h)\in(0,h)$,
 \begin{equation*}
\begin{aligned}
\nabla L(\bw) &=  \frac{d}{d\bw}\E[L_n(\bw)] \\
&= \lim_{h\to0}\frac{1}{h}\left\{\E[L_n(\bw+h)] - \E[L_n(\bw)]\right\}\\
&= \lim_{h\to0}\E\left\{\frac{L_n(\bw+h) - L_n(\bw)}{h}\right\} = \lim_{h\to0}\E\left\{\nabla L_n(\bw+\tau(h))\right\}.
\end{aligned}
\end{equation*}
By the continuity of $\nabla L_n$ and the dominated convergence theorem, 
\begin{equation*}
\lim_{h\to0}\E\left\{\nabla L_n(\bw+\tau(h))\right\} = \E\left\{\lim_{h\to0}\nabla L_n(\bw+\tau(h))\right\}  = \E\left\{\nabla L_n(\bw)\right\}.
\end{equation*}
Hence, $\boldsymbol{\epsilon}_k$ has mean 0. Since the independent and uniform sampling for the mini-batch $B_k$, we have  $\text{Var}[\boldsymbol{\epsilon}_k]  = \bssigma^2(\bw)$ as desired.

We remark that a different view of sampling distribution has been adopted in the literature, for example, \cite{li2017} and 
\cite{Jastrzebski2017}, where the expectation and variance are taken with respect to the sampling distribution of drawing the mini-batch $B_k$ from  $\{1,\ldots,N\}$. On the contrary, we use the sampling distribution with respect to the joint distribution of the underlying population, since our interest is the risk function
$L(\cdot)$ instead of the sample average loss 
\begin{equation*}
\frac{1}{N}[L_1(\cdot)+\cdots+L_N(\cdot)],
\end{equation*}
and we regard the training data only a subset of the underlying population.

\subsection{Proof of Lemma \ref{lem:fokkerplank}}
\label{sec:proofoflemfokkerplank}

We first consider a special case that $\beta(\bw) \equiv\beta$ is a constant and derive the Fokker-Planck equation by following Kolpas et al. \cite{kolpas2007coarse}.
If $\bW(t)=W(t)\in\R$, $W(t)$ is a Markov process and the Chapman-Kolmogorov equation gives  the conditional probability density function  for any $t_1\leq t_2\leq t_3$,
\begin{equation*}
p\left(W(t_3)|W(t_1)\right) = \int_{-\infty}^{+\infty} p\left(W(t_3)|W(t_2)=w\right)p\left(W(t_2)=w|W(t_1)\right)dw.
\end{equation*}
Denote the integral
\begin{equation}
\label{eqn:defofih}
I(h) = \int_{-\infty}^{+\infty}h(w)\partial_t p(w,t|W)dw,
\end{equation}
where $h(w)$ is a smooth function with compact support. Observe that
\begin{equation*}
\int_{-\infty}^{+\infty}h(w)\partial_t p(w,t|W)dw = \lim_{\Delta t\to 0 }\int_{-\infty}^{+\infty}
h(w)\left(\frac{p(w,t+\Delta t|W) - p(w,t|W)}{\Delta t}\right)dw.
\end{equation*}
Letting $Z$ be an intermediate point between $w$ and $W$. Applying the Chapman-Kolmogorov identity on the right hand side yields
\begin{equation*}
\lim_{\Delta t\to 0 }\frac{1}{\Delta t}\left(\int_{-\infty}^{+\infty}h(w)\int_{-\infty}^{+\infty}p(w,\Delta t|Z)p(Z,t|W)dZdw-\int_{-\infty}^{+\infty}h(w)p(w,t|W)dw\right).
\end{equation*}
By changing the order of integrations in the first term and letting $w$ approach $Z$ in the second term, we obtain that
\begin{equation*}
\lim_{\Delta t\to 0}\frac{1}{\Delta t}\left(\int_{-\infty}^{+\infty}p(Z,t|W)\int_{-\infty}^{+\infty}p(w,\Delta t|Z)(h(w)-h(Z))dwdZ\right).
\end{equation*}
Expand $h(w)$ as a Taylor series about $Z$, we can write the above integral as
\begin{equation*}
\lim_{\Delta t\to 0}\frac{1}{\Delta t}\left(\int_{-\infty}^{+\infty}p(Z,t|W)\int_{-\infty}^{+\infty}p(w,\Delta t|Z)\sum_{n=1}^{\infty}h^{(n)}(Z)\frac{(w-Z)^n}{n!}\right)dwdZ.
\end{equation*}
Now we define the function 
\begin{equation*}
D^{(n)}(Z)=\frac{1}{n!}\frac{1}{\Delta t}\int_{-\infty}^{+\infty}p(w,\Delta t|Z)(w-Z)^ndw.
\end{equation*}
We can write the integral $I(h)$ defined in (\ref{eqn:defofih}) as
\begin{equation*}
\int_{-\infty}^{+\infty}h(w)\partial_t p(w,t|W)dw = \int_{-\infty}^{+\infty}p(Z,t|W)\sum_{n=1}^{\infty}D^{(n)}(Z)h^{(n)}(Z)dZ.
\end{equation*}
Taking the integration by parts $n$ times gives
\begin{equation*}
\partial_t p(w,t) = \sum_{n=1}^{\infty}-\frac{\partial^n}{\partial Z^n}\left[D^{(n)}(Z)p(Z,t|W)\right].
\end{equation*}
Let $D^{(1)}(w) = -L(w)$, 
$D^{(2)}(w) = -\gamma(t)\beta/[2M(t)]$ and $D^{(n)}(w) = 0$ for all $n\geq 3$. Then the above equation yields
\begin{equation*}
\partial_t p(w,t) =  \frac{\partial}{\partial w}\left[\nabla L(w)p(w,t)\right] + \frac{\partial}{\partial w^2}\left[\frac{\gamma(t)\beta}{2M(t)}p(w,t)\right],
\end{equation*}
which is the Fokker-Planck equation in one variable. 
For the multidimensional case that $\bW=(W_1,W_2,\ldots,W_p)\in\R^p$, we similarly generalize the above procedure  to get 
\begin{equation}
\label{eqn:fkpconstbeta}
\begin{aligned}
\partial_t p(\bw,t) & =  \sum_{i=1}^p\frac{\partial}{\partial w_i}\left[\nabla L(\bw)p(\bw,t)\right] +\sum_{i=1}^p \frac{\partial^2}{\partial w_i^2}\left[\frac{\gamma(t)\beta}{2M(t)}p(\bw,t)\right]\\
& = \nabla\cdot\left(\nabla L(\bw)p +\frac{\gamma(t)\beta}{2M(t)}\nabla p\right).
\end{aligned}
\end{equation}
Since $\bW(0)=\bw_0$, $p(\bw,0)=\delta(\bw_0)$. This completes the derivation of the Fokker-Planck equation for constant $\beta(\bw) =\beta$.

For deriving (\ref{eqn:probdensitypthetatw}) with general $\beta(\bw)$, we can simply apply (\ref{eqn:fkpconstbeta}) together with the fact that 
\begin{equation*}
\nabla\left[\frac{\gamma(t)\beta(\bw)}{2M(t)}p\right] = \nabla\left[\frac{\gamma(t)\beta(\bw)}{2M(t)}\right]p + \frac{\gamma(t)\beta(\bw)}{2M(t)}\nabla p.
\end{equation*}
This completes the proof.

\section{Proofs for Section  \ref{sec:mainresults}}
\label{sec:proofofsecconvergence}

\subsection{Discussion on Main Assumptions (A.1)--(A.3)}
\label{sec:disonassump}

We show that Assumptions (A.1)--(A.3) hold for the squared loss and the  regularized  mean cross entropy loss. 
Denote by $\{(\bx_n,y_n),1\leq n\leq N\}$ the set of training data. 
Without loss of generality, let $\text{Var}[y_n|\bx_n] = 1$.
First, we consider the squared loss with the corresponding risk function
\begin{equation*}
L(\bw) = (\bw-\bw^0)^\top\E[\bx_n\bx_n^\top](\bw-\bw^0)+1,
\end{equation*} 
where $\bw^0$ is the true parameter vector.
Since $\text{Var}[\nabla L_n(\bw)]\equiv\bssigma^2(\bw)$ is positive definite, we have
\begin{equation}
\label{eqn:a1lst}
\begin{aligned}
\lim_{\|\bw\|\to+\infty}L(\bw)  & \geq \lim_{\|\bw\|\to+\infty} \lambda_{\min}\{\E[\bx_n\bx_n^\top]\}\|\bw-\bw^0\|^2 +1\\
&  \geq \lim_{\|\bw\|\to+\infty} \lambda_{\min}\{\E[\bx_n\bx_n^\top]\}[\|\bw\|^2/2-\|\bw^0\|^2/2] +1 = +\infty,
\end{aligned}
\end{equation}
where $\lambda_{\min}\{\cdot\}$ denotes the minimal eigenvalue.  Note that
\begin{equation*}
\begin{aligned}
\int e^{-L(\bw)} d\bw  & = \int \exp\left(-(\bw-\bw^0)^\top\E[\bx_n\bx_n^\top](\bw-\bw^0)-1\right)d\bw\\
&  \leq\int \exp\left(- \lambda_{\min}\{\E[\bx_n\bx_n^\top]\}[\|\bw\|^2/2-\|\bw^0\|^2/2]-1\right)d\bw < +\infty.
\end{aligned}
\end{equation*}
Hence, Assumption (A.1) holds. To prove (A.2),  note that 
\begin{equation*}
\|\nb L(\bw)\|^2/2 = 2 (\bw-\bw^0)^\top\{\E[\bx_n\bx_n^\top]\}^2(\bw-\bw^0), \quad \text{Tr}(\nb^2 L(\bw)) = \text{Tr}\{\E[\bx_n\bx_n^\top]\}.
\end{equation*}  
Similar to (\ref{eqn:a1lst}), we can prove  that 
\begin{equation*}
\lim_{\|\bw\|\to +\infty}  \left\{\|\nb L(\bw)\|^2/2 - \text{Tr}(\nb^2 L(\bw))\right\} = + \infty, \quad \lim_{\|\bw\|\to +\infty}  \left\{\text{Tr}(\nb^2 L(\bw))/\|\nb L(\bw)\|^2\right\} = 0.
\end{equation*} 
This finishes the proof for Assumption (A.2). Finally, (A.3) can be shown similarly by following the proof for (A.2) and we omit the details.

Next, we consider the mean cross entropy loss  with the $l_2$-penalty for the logistic regression. Without loss of generality, we consider the binary  classification: 
\begin{equation*}
L(\bw) = \E[-y_n\log\widehat{y}_n -(1-y_n)\log(1-\widehat{y}_n)]+\lambda\|\bw\|^2
\end{equation*} with $\widehat{y}_n = (1+e^{-\bw\cdot\bx_n})^{-1}$. Note that
\begin{equation*}
\lim_{\|\bw\|\to+\infty}L(\bw)  \geq \lambda\|\bw\|^2 = +\infty, \quad \int e^{-L(\bw)} d\bw   \leq \int e^{-\lambda\|\bw\|^2}d\bw<+\infty
\end{equation*}
which proves (A.1). For (A.2), since
\begin{equation*}
\nabla L(\bw) = \E[-\bx_ny_n + \bx_n/(1+e^{-\bw\cdot\bx_n})]+2\lambda\bw,
\end{equation*}
 and 
 \begin{equation*}
 \text{Tr}(\nb^2 L(\bw)) = \E\left[\frac{e^{-\bw\cdot \bx_n}}{(1+e^{-\bw\cdot\bx_n})^2}\text{Tr}(\bx_n\bx_n^\top)\right]+2\lambda d,
 \end{equation*} 
we have
  \begin{equation*}
 \|\nb L(\bw)\|^2/2 - \text{Tr}(\nb^2 L(\bw))\to\infty, \quad
  \text{Tr}(\nb^2L(\bw))/\|\nb L(\bw)\|^2\to0,   \text{ as } \|\bw\|\to\infty.
   \end{equation*} 
Similarly, Assumption  (A.3) can be verified as by following the proof for (A.2). 

\subsection{Proof of Lemma \ref{lem:stationbeta}}
\label{sec:proflemstationbeta}

By Assumption (A.1),  the density function $p_\infty(\bw)\equiv\kappa e^{-2M(\infty)L(\bw)/[\gamma(\infty)\beta]}$ is well-defined. Moreover, $p_\infty(\bw)$ satisfies 
 \begin{equation*}
 \nabla\cdot\left[\nabla \left(L(\bw)+\frac{\gamma(\infty)\beta}{2M(\infty)}\right)p_\infty(\bw) + \frac{\gamma(\infty)\beta}{2M(\infty)}\nabla p_\infty(\bw)\right]=0.
 \end{equation*} 
 Hence,  $p_\infty(\bw)$ is a stationary solution to Fokker-Planck equation  (\ref{eqn:probdensitypthetatw}) by letting $ \partial_tp(\bw,t) =0$.


\subsection{Proof of Theorem \ref{thm:sigm4.1proof}}
\label{sec:proofofthm:sigm4.1proof}

Parallel to the notation $p_\infty(\bw) =  \kappa \exp(- \frac{2M(\infty)L(\bw)}{\g(\infty)\b})$ in Lemma \ref{lem:stationbeta}, we define 
\begin{equation*}
\hp(\bw,t)\equiv \k(t) \exp\left(-\eta(t)L(\bw)\right),
\end{equation*} 
where 
\begin{equation}
\label{eqn:defofeta}
\eta(t) \equiv 2M(t)/[\gamma(t)\b],
\end{equation} and $\k(t)$  is a time-dependent normalization factor such that 
\begin{equation*}
\int\hp(\bw,t)d\bw=1.
\end{equation*} 
We can rewrite (\ref{eqn:probdensitypthetatw}) as
\begin{equation}
	\pt_tp = \frac{1}{\eta}\nb_\bw\cdot\(\hp \nb_\bw\(\frac{p}{\hp}\)\).
	\label{eqn: fp_iso}
\end{equation}
Let 
\begin{equation*}
\d(t,\bw)\equiv \frac{\kappa(t)}{\kappa}\exp\left(L(\bw)\({\eta(\infty)} - {\eta(t)}\)\right).
\end{equation*}
Then
\begin{equation*}
\hp(t,\bw)  = \pinf(\bw)\d(t,\bw).
\end{equation*}
Denote by $h(\bw,t)$ the scaled distance  between $p(\bw,t)$ and $\pinf(\bw)$:
\begin{equation*}
	h(\bw,t)\equiv\frac{p(\bw,t) - p_\infty(\bw)}{\sqrt{p_\infty(\bw)}},
\end{equation*}
which satisfies the following equation:
\begin{equation}
\label{eqn:pertsolu}
\begin{aligned}
	\pt_t h =& \frac{1}{\eta\sqrt{\pinf}}\nb_\bw\cdot\l[\hp \,\nb_\bw\(\frac{1}{\d} + \frac{h}{\sqrt{\pinf} \d} \)\r]\\
	=& \frac{1}{\eta\sqrt{\pinf}}\nb_\bw\cdot\l[\pinf \(\nb_\bw L\hd + \nb_\bw L\hd\(\frac{h}{\sqrt{\pinf}}\) +\nb_\bw\(\frac{h}{\sqrt{\pinf}}\) \)\r].
\end{aligned}
\end{equation}
Here, $\hd$ is defined as $\hd(t) =\eta(t) - \eta(\infty)$, where $\eta(\infty) = \lim_{t\to\infty}\eta(t)$.
We multiply $h$ to the both sides of (\ref{eqn:pertsolu}) and integrate them over $\bw$. Using the integration by parts, we can obtain 
\begin{equation}
\label{eqn:decparth2}
\begin{aligned}
	\frac{1}{2}\pt_t\ll h \rl^2 =& \frac{\hd}{\eta} \underbrace{\int \frac{h}{\sqrt{\pinf}}\nb_\bw\cdot\(\pinf \nb_\bw L\r)d\bw}_{I} + \frac{\hd}{\eta} \underbrace{\int\frac{1}{2}\ll\frac{h}{\sqrt{\pinf}}\rl^2 \nb_\bw\cdot\(\pinf\nb_\bw L\)d\bw}_{II}\\
	&- \frac{1}{\eta} \underbrace{\int \pinf\ll\nb_\bw\(\frac{h}{\sqrt{\pinf}}\)\rl^2d\bw}_{III}.
\end{aligned}
\end{equation}
We study the parts $I, II, III$ in the right-hand side of above equation separately. 

For the part $I$,  note that 
\begin{equation*}
	\nb_\bw\cdot\(\pinf\nb_\bw L\) = \pinf\(\nb_\bw\cdot\nb_\bw L  - \eta(\infty)\ll \nb_\bw L\rl^2 \).
\end{equation*}
Hence, Assumption (A.3) yields that
\begin{equation*}
	\lv\nb_\bw\cdot\(\pinf\nb_\bw L\)\rv \leq \pinf^{2/3}\max\{1 , \eta(\infty)\}M(\infty),
\end{equation*}
which implies that an upper bound of part $I$ in (\ref{eqn:decparth2}):
\begin{equation*}
	I \leq  \frac{\max\{1 , \eta(\infty)\}M(\infty)}{2}\(\ll h \rl^2 
	+ \int \pinf^{1/3}d\bw \).
\end{equation*}

For the part $II$, note that  Assumption (A.3) gives
\begin{equation*}
\lim_{\ll \bw \rl\to \infty} \frac{\nb_\bw\cdot\nb_\bw L}{2\eta(\infty)\ll \nb_\bw L\rl^2 } = 0,
\end{equation*} 
which together with the assumption (A.2) implies that 
\begin{equation*}
\lim_{\ll \bw \rl\to\infty} \ll \nb_\bw L\rl^2 \to +\infty.
\end{equation*}
 Thus, there exists a constant $R$, such that
\begin{equation*}
	 \nb_\bw\cdot\nb_\bw L  - 2\eta(\infty)\ll \nb_\bw L\rl^2 \leq \eta(\infty), \qd \eta(\infty)\ll \nb_\bw L\rl^2 \geq \eta(\infty),\qd  \text{for }\forall \ll \bw \rl > R. 
\end{equation*}
Hence, 
\begin{equation*}
	 \nb_\bw\cdot\nb_\bw L  - \eta(\infty)\ll \nb_\bw L\rl^2 \leq 0, \qd \text{for }\forall \ll \bw \rl > R.
\end{equation*}
By the continuity of $L(\bw)$, there exists a constant $C_2$ such that
\begin{equation*}
	 \lv\nb_\bw\cdot\nb_\bw L  - \eta(\infty)\ll \nb_\bw L\rl^2\rv \leq C_2, \qd \text{for }\forall \ll \bw \rl <R.
\end{equation*}
Therefore, we have the following upper bound for the part $II$ in (\ref{eqn:decparth2}):
\begin{equation*}
	\lv II \rv \leq \frac{C_2}{2}  \ll h \rl^2 . 
\end{equation*}
By combining the estimates for the parts $I$ and $II$, we have
\begin{equation*}
	I+II \leq  C_1 \ll h \rl^2 + C_1, \label{eqn:part2est}
\end{equation*}
where $C_1 = \frac{1}{2}\max\{1 , \eta(\infty)\}\max\l\{ \int \pinf^{1/3}d\bw, 1+C_2/2\r\}M(\infty)$.

For the part $III$, note that Assumption (A.2) implies  the following Poincar\'e inequality (see, e.g., \cite{pavliotis2014stochastic}), 
\begin{equation}
\label{eqn:poincareintw}
	  \int \ll\nb_\bw\(\frac{h}{\sqrt{\pinf}}\)\rl^2  \pinf \,d\bw  \geq C_P\int  \(\frac{h}{\sqrt{\pinf}} - \int h\sqrt{\pinf} d\bw \)^2 \pinf \,d\bw.
\end{equation}
We need to show that 
\begin{equation}
\label{eqn:inthsqrtpinf0}
\int h\sqrt{\pinf}\,d\bw = 0.
\end{equation} 
The (\ref{eqn:inthsqrtpinf0}) can be proven using  the conservation of mass. In particular,  if we integrate (\ref{eqn: fp_iso}) over $\bw$ and use the integration by parts, 
\begin{equation*}
	  \pt_t\(\int  p(\bw,t)\, d\bw\)  = 0, 
\end{equation*}
which implies $\int h\sqrt{\pinf} \,d\bw = \int p\, d\bw - \int \pinf \,d\bw = 0$. 
Combining (\ref{eqn:poincareintw}) with (\ref{eqn:inthsqrtpinf0}) gives a lower bound for the part $III$:
\begin{equation*}
	 III  \geq C_P\ll h \rl^2.\label{eqn:part1est}
\end{equation*}
Combining (\ref{eqn:part2est}) and (\ref{eqn:part1est}) gives
\begin{equation}
	  \frac{1}{2}\pt_t\ll h \rl^2 + \frac{C_P}{\eta} \ll h \rl^2 \leq \frac{C_1\hd}{\eta} \(\ll h \rl^2 +1\)
	  \label{eqn:energyest}
\end{equation}
Since $\eta(t) \to \eta(\infty) >0$ as $t\to\infty$, there exists some $T$ large enough and for $\forall t>T$, 
\begin{equation}
	  \hd = \lv \eta(t) - \eta(\infty) \rv \leq\min\l\{\frac{\eta(\infty)}{3}, \frac{C_P}{3C_1}\r\}.
	  \label{eqn:condonT}
\end{equation}
Plugging $\hd \leq C_P/3C_1$ into (\ref{eqn:energyest}), we have
\begin{equation}
	  \frac{1}{2}\pt_t\ll h \rl^2 + \frac{2C_P}{3\eta} \ll h \rl^2 \leq \frac{C_P}{3\eta}, \qd \text{for }\forall t>T .
	  \label{mid_1}
\end{equation}
Note that (\ref{eqn:condonT}) also implies that $2\eta(\infty)/3 \leq \eta(t) \leq 4\eta(\infty)/3$. Thus,
\begin{equation*}
	 \frac{2C_P}{3\eta}\geq\frac{C_P}{2\eta(\infty)}, \qd  \frac{C_P}{3\eta} \leq \frac{C_P}{2\eta(\infty)}.
\end{equation*}
Plugging back to (\ref{mid_1}), we arrive at 
\begin{equation*}
	  \frac{1}{2}\pt_t\ll h \rl^2 + \frac{C_P}{2\eta(\infty)} \ll h \rl^2 \leq \frac{C_P}{2\eta(\infty)}, \qd \text{for }\forall t>T .
\end{equation*}
Integrating the above equation from $T$ to $t>T$, we have
\begin{equation*}
	  \ll h(t) \rl^2 \leq \(\ll h(T) \rl^2 + \frac{C_P}{\eta(\infty)}(t-T)\) - \frac{C_P}{\eta(\infty)}\int_T^t  \ll h(s) \rl^2ds.
\end{equation*}
By Gronwall's Inequality, we finally get 
\begin{equation*}
	\ll h(t)\rl^2 \leq \(\frac{C_P}{\eta(\infty)}(t-T) + \ll h(T) \rl^2\) \exp\left(-\frac{C_P}{\eta(\infty)}(t-T)\right).
\end{equation*}
This completes the proof. 

\subsection{Quantification of $T$ in Theorem \ref{thm:sigm4.1proof}}
\label{sec:quantificT}

We quantify $T$ by giving a condition that a minimum $T$ should satisfy. From the proof in Section \ref{sec:proofofthm:sigm4.1proof}, it is clear that $T$ should be large enough such that for all $t>T$, 
\begin{equation*}
	  \lv \eta(t) - \eta(\infty) \rv \leq\min\l\{\frac{\eta(\infty)}{3}, \frac{C_P}{3C_1}\r\},
	  \end{equation*}
	  where $\eta(t)$ is defined in (\ref{eqn:defofeta}) and $\eta(\infty) = \lim_{t\to\infty}\eta(t)$, and 
	  \begin{equation*}
	  C_1 = \frac{M}{2}\max\{1 , \eta(\infty)\}\max\l\{ \int \pinf^{1/3}d\bw, 1+\frac{C_2}{2}\r\},
\end{equation*}
and $C_2>0$ is an upper bound for $\lv\nb_\bw\cdot\nb_\bw L  - \eta(\infty)\ll \nb_\bw L\rl^2\rv$ in the bounded domain $\{\ll \bw \rl <R\}$ such that
\begin{equation*}
	 \nb_\bw\cdot\nb_\bw L  - \eta(\infty)\ll \nb_\bw L\rl^2 \leq \left\{
	 \begin{aligned}
	 &0, \qd \text{for }\forall \ll \bw \rl > R,\\
	 &C_2, \qd \text{for }\forall \ll \bw \rl <R.
	 \end{aligned}
	 \right.
\end{equation*}


\subsection{Proof of Theorem \ref{thm:probmainresult4111}}
\label{eqn:proofofthfinal}

Denote by $P_\e(\check{\bw})  = \P(\ll \bW(\infty) - \check{\bw}\rl\leq \epsilon)$
 the probability that $\bW(\infty)$ is trapped in an $\e$-neighborhood of the minimum $\check{\bw}$. Recall the probability density function of $\bW(\infty)$ is $\pinf(\bw)$. Then
\begin{equation*}
\begin{aligned}
P_\e(\check{\bw}) & = \int_{\ll\bw - \check{\bw}\rl^2\leq\e^2}\kappa e^{-\eta(\infty) L(\bw)}d\bw \\
& = \int_{\ll\bw - \check{\bw}\rl^2\leq\e^2} \kappa \exp\left(-\eta(\infty)[L(\check{\bw})+(\bw-\check{\bw})'\nb^2L(\check{\bw})(\bw-\check{\bw})+o\{(\bw - \check{\bw})^2\}]\right)d\bw,
\end{aligned}
\end{equation*}
where  $\eta(t)$ is defined in (\ref{eqn:defofeta}) and $\eta(\infty) = \lim_{t\to\infty}\eta(t)$.
Since $\check{\bw}$ is a local minimum of $L(\bw)$,  $\nb^2L(\check{\bw})$ is positive definite. There exists an orthogonal matrix $O$ and diagonal matrix $\L$ such that $\nb^2L  = O' \L O$. For simplicity, we assume that $\nb^2L = \L = \text{diag}(\lam_{\min}, \cdots, \lam_d)$.  Then
\begin{equation*}
\begin{aligned}
&\lim_{\e\to0}P_\e(\check{\bw})\\
& = \lim_{\e\to0}\left[\kappa e^{-\eta(\infty) L(\check{\bw})} \int_{\ll \bw \rl^2\leq \e^2} \prod_{j=1}^d e^{-\eta(\infty) \lam_j w_j}    d\bw\right]      e^{\eta(\infty)\epsilon^2}\\
&= \lim_{\e\to0}\left[\kappa e^{-\eta(\infty) L(\check{\bw})}\prod_{j=1}^d\frac{1}{\sqrt{\eta(\infty)\lam_j}}\int_{-\epsilon\sqrt{\eta(\infty) \lam_j}}^{\epsilon\sqrt{\eta(\infty) \lam_j}}e^{-w^2}dw\right]e^{\eta(\infty)\epsilon^2}\\
&=\lim_{\e\to0}\l[\kappa \eta(\infty)^{-d/2} e^{-\eta(\infty) L(\check{\bw})}\prod_{j=1}^d\frac{1}{\sqrt{\lam_j}}\(\Phi\(\epsilon\sqrt{\eta(\infty) \lam_j}\) - \Phi\(-\epsilon\sqrt{\eta(\infty) \lam_j}\)\)\r]e^{\eta(\infty)\epsilon^2},
\end{aligned}
\end{equation*}
where $\Phi(\cdot)$ is the cumulative density function for standard normal distribution. The first equality is from the change of variable by writing $ \bw - \check{\bw}$ as $\bw$. The second equality  is from changing $\eta(\infty)\lam_j\bw_j$ to $\bw_j$. Using the approximation of the cumulative density function in P{\'o}lya \cite{polya1945remarks}, we can simplify the above equation as
\begin{equation*}
\begin{aligned}
\lim_{\e\to0}P_\e(\check{\bw})  =&\lim_{\e\to0}\l[ \frac{\kappa e^{-2\eta(\infty) L(\check{\bw})}}{\eta(\infty)^{d/2}}\prod_{j=1}^d\sqrt{\frac{1-e^{-\e^2\eta(\infty)\lam_j/\pi}}{\lam_j}}\r]e^{\eta(\infty)\epsilon^2} \\
=&\frac{\kappa e^{-2\eta(\infty) L(\check{\bw})}}{\eta(\infty)^{d/2}|\nb^2L(\check{\bw})|} \lim_{\e\to0}\l[e^{\eta(\infty)\epsilon^2}\prod_{j=1}^d\sqrt{1-e^{-\e^2\eta(\infty)\lam_j/\pi}}\r].
\end{aligned}
\end{equation*}
We complete the proof.


\subsection{Proof of Equation \ref{eqn:ratioofprob}}
\label{eqn:proofofeqnratio}
Denote by $\lambda_j^k$'s are eigenvalues of  the Hessian  $\nb^2L(\check{\bw}_k)$, $k=1,2$ and $j\geq 1$.
By Theorem \ref{thm:probmainresult4111} and $L(\check{\bw}_1)=L(\check{\bw}_2)$, we have that
\begin{equation*}
\begin{aligned}
&\lim_{\e\to0}\frac{\P(|\bW(\infty) - \check{\bw}_1|\leq \epsilon) }{\P(|\bW(\infty) - \check{\bw}_2|\leq \epsilon)}  
= \frac{\lv\nb^2L(\check{\bw}_2)\rv}{\lv\nb^2L(\check{\bw}_1)\rv}\sqrt{\lim_{\e\to0}\prod_{j=1}^d\frac{1-\exp\left(-\frac{\e^2\eta(\infty)\lam^1_j}{\pi}\right)}{1-\exp\left(-\frac{\e^2\eta(\infty)\lam^2_j}{\pi}\right)}}\\
=&\frac{\lv\nb^2L(\check{\bw}_2)\rv}{\lv\nb^2L(\check{\bw}_1)\rv}\sqrt{\lim_{\e\to0}\prod_{j=1}^d\frac
{\lam^1_j\exp\left(-\frac{\e^2\eta(\infty)\lam^2_j}{\pi}\right)}{\lam^2_j\exp\left(-\frac{\e^2\eta(\infty)\lam^2_j}{\pi}\right)}}
= \frac{\lv\nb^2L(\check{\bw}_2)\rv}{\lv\nb^2L(\check{\bw}_1)\rv}\sqrt{\prod_{j=1}^d\frac
{\lam^1_j}{\lam^2_j}} = \sqrt{ \frac{\lv\nb^2L(\check{\bw}_2)\rv}{\lv\nb^2L(\check{\bw}_1)\rv}},
\end{aligned}
\end{equation*}
where  $\eta(t)$ is defined in (\ref{eqn:defofeta}) and $\eta(\infty) = \lim_{t\to\infty}\eta(t)$.


\section{Proofs for Section \ref{sec: model}}
\subsection{Derivation of SDE for MSGD}
\label{sec: proof of msde}
For constant learning rate and batch size: $\gamma_k\equiv \gamma, M_k\equiv M$, we rewrite the MSGD as
\begin{equation*}
\begin{aligned}
	&\frac{\bz_{k+1}}{\sqrt{\g}} = \frac{\bz_k}{\sqrt{\g}} + \sqrt{\g}\(- \frac{1 - \xi}{\g} \bz_k - \nb L(\bw_k)\) + \sqrt{\g}\(\nb L(\bw_k) - \(\frac{1}{M}\sum_{n\in B_k}\nb L_n(\bw_k)\)\)\\
	&\bw_{k+1} = \bw_k + \frac{\bz_{k+1}}{\sqrt{\g}}\sqrt{\g}.
\end{aligned}
\end{equation*}
Let $\bv_k = \bz_k/\sqrt{\g}$. We have the approximation for MSGD
\begin{equation*}
\begin{aligned}
	&\bv_{k+1} - \bv_k =  - \frac{1 - \xi}{\sqrt{\g}} \bv_k \sqrt{\g} - \nb L(\bw_k)  \sqrt{\g} +  \frac{\g^{1/4}}{\sqrt{M}} \sqrt{\b}\nb^2B_t,\\
	&\bw_{k+1} - \bw_k = \bv_{k+1}\sqrt{\g},
\end{aligned}
\end{equation*}
where $\b(\bw)$ is the  covariance function defined in (\ref{eqn:isotconv}). Hence, MSGD is approximated as the Euler-Maruyama discretization for the following SDE, 
\begin{equation*}
\left\{
\begin{aligned}
&d\bV(t) = -\nb L(\bW(t)) dt - \frac{1 - \xi}{\sqrt{\g}} \bV(t) dt + \frac{\g^{1/4}}{\sqrt{M}}\sqrt{\b(\bW(t))} d\bB(t),\\
	&d \bW(t) = \bV(t)dt,
\end{aligned}
\right.
\end{equation*}
where  $\bv_k \approx \bV({k\sqrt{\g}})$, $\bw_k \approx \bW({k\sqrt{\g}})$.

\subsection{Proof of Lemma  \ref{lemma: pdf msgd} }
\label{sec: proof of lemmsgd}
We give a formal derivation, which is similar to the procedure in Pavliotis \cite{pavliotis2014stochastic}. 
Let $\phi(\cdot,\cdot)$ be any bivariate function in $C^\infty$  with a compact support. Using the It$\h{o}$'s formula, 
\begin{equation*}
\begin{aligned}
    & d\phi(\bW(t), \bV(t))   = \frac{\g^{1/4}}{\sqrt{M}}\sqrt{\b} \nb_\bw\cd\nb_\bw\phi d\bB(t)\\
    &\quad\quad +  \l(\bV(t)\cdot\nb_\bw\phi + \left(-\nb L(\bW(t)) -  \frac{1 - \xi}{\sqrt{\g}} \bV (t)\right)\cdot \nb_\bv \phi + \frac{\g^{1/2}}{2M}\b(\bW(t)) \nb_\bv \cdot\nb_\bv \phi\r)dt.   
\end{aligned}
\end{equation*}
By taking the expectation of the above equation and integrating it over the range $[t, t+h]$, we obtain that
\begin{equation*}
\begin{aligned}
    &\frac{1}{h}\E\l(\phi(\bW(t+h), \bV(t+h)) - \phi(\bW(t), \bV (t))\r) \\
    =& \frac{1}{h}\int_t^{t+h} \E \l(\bV (s)\cdot\nb_\bw\phi + \left(-\nb L(\bW(s)) -  \frac{1 - \xi}{\sqrt{\g}} \bV (s)\right) \cdot\nb_\bv \phi + \frac{\gamma^{1/2}\beta(\bW(s))}{2M} \nb_\bv \cdot\nb_\bv \phi\r) ds .    
\end{aligned}
\end{equation*}
Let $\p(\bw,\bv,t)$ be the joint probability density function of $(\bW(t), \bV(t))$.  The above equation can also be written as
\begin{equation*}
\begin{aligned}
    &\frac{1}{h}\int \phi(\bw, \bv ) \l( \p(\bw,\bv,t+h) - \p(\bw,\bv,t)\r)\,d\bw\,d\bv  \\
    =& \frac{1}{h}\int_t^{t+h} \int \l(\bv \cdot\nb_\bw\phi + \left(-\nb L(\bw) -  \frac{1 - \xi}{\sqrt{\g}} \bv \right) \cdot\nb_\bv \phi + \frac{\gamma^{1/2}\beta(\bw)}{2M} \nb_\bv \cdot\nb_\bv \phi \r) \p(\bw,\bv,s)\,d\bw\,d\bv \, ds.
\end{aligned}
\end{equation*}
Then, using the integration by parts and  letting $h\to0$ gives
\begin{equation*}
\begin{aligned}
 &   \int \phi(\bw, \bv) \pt_t\p\,d\bw\,d\bv \\
 &= \int \phi\l(-\bv \cdot\nb_\bw\p +\nb L(\bw)\cdot\nb_\bv \p +\nb_\bv \cdot \left( \frac{1 - \xi}{\sqrt{\g}} \bv \p\right) + \frac{\gamma^{1/2}\beta(\bw)}{2M}\nb_\bv \cdot\nb_\bv \p \r) \,d\bw\,d\bv ,
 \end{aligned}
\end{equation*}
which is satisfied for any test functions. Therefore, the density function $\p(\bw,\bv,t)$ satisfies
\begin{equation*}
    \pt_t\p + \bv \cdot\nb_\bw\p -\nb L(\bw)\cdot\nb_\bv \p =  \nb_\bv \cdot\( \frac{1 - \xi}{\sqrt{\g}} \bv \p + \frac{\gamma^{1/2}\beta(\bw)}{2M} \nb_\bv \p\),
\end{equation*}
which agrees with (\ref{eq: origin pde}).

Next, we can verify that $\psi_\infty(\bw,\bv ) $ is a stationary solution of the Vlasov-Fokker-Planck equation (\ref{eq: origin pde}) by a direct calculation as Appendix \ref{sec:proflemstationbeta}.


\subsection{Discussion on Assumption (A.4)}
\label{sec:discassumpa4}

We show that Assumption (A.4) holds for the squared loss and the  regularized  mean cross entropy loss. 
Denote by $\{(\bx_n,y_n),1\leq n\leq N\}$ the set of training data. 
Without loss of generality, let $\text{Var}[y_n|\bx_n] = 1$.
For the squared loss, 
\begin{equation*}
    \tL(\bw) = (\bw-\bw^0)^\top\E[\bx_n\bx_n^\top](\bw-\bw^0)+1 - \frac{1}{2}C_L^2\|\bw\|^2,
\end{equation*}
where $\bw^0$ is the true parameter vector. By a direct calculation,
\begin{equation*}
    \nb^2\tL(\bw) = 2\E[\bx_n\bx_n^\top] - C^2_L.
\end{equation*}
Since the eigenvalues of the design matrix $\E[\bx_n\bx_n^\top]$ are bounded, the eigenvalues of $ \nb^2\tL(\bw) $ are bounded for any $C_L$. Hence, Assumption (A.4) holds for the squared loss. 

Nest, we consider the regularized mean cross entropy loss for the logistic regression. Similar to Appendix \ref{sec:disonassump}, letting $C_L = \sqrt{2\lambda}$ yields that
\begin{equation*}
    \tL(\bw) = \E[-y_n\log\widehat{y}_n -(1-y_n)\log(1-\widehat{y}_n)].
\end{equation*}
The $(i,j)$th entry of the Hessian $\nb^2L(\bw)$ is
\begin{equation*}
(\nb^2L(\bw))_{ij} = \E\left[x_{ni}x_{nj}\frac{e^{-\bw\cdot\bx_n}}{(1+e^{-\bw\cdot\bx_n})^2}\right],
\end{equation*}
where $x_{ni}$ is the $i$th element of $\bx_n$.
Then, 
\begin{equation*}
(\nb^2L(\bw))_{ij} \to 0\quad \text{as } \|\bw\|\to\infty,
\end{equation*}
which implies that there exists finite constant $b_{ij}>0$ such that $\|(\nb^2L(\bw))_{ij} \|_{\infty}\leq b_{ij}$ and the largest row sum of the matrix $\{\|(\nb^2L)_{ij}\|_{\infty}\}_{1\leq i,j\leq d}$ is upper bounded by $b \equiv \max_i(\sum_{j}b_{ij})$. 
Since the largest eigenvalue of a non-negative matrix is upper bounded by its largest row sum, 
the eigenvalues of $\{\|(\nb^2L)_{ij}\|_{\infty}\}_{1\leq i,j\leq d}$ are bounded by $b$. Hence, Assumption (A.4) also holds for the regularized mean cross entropy loss.

\subsection{Proof of Theorem \ref{thm: main thm} }
\label{sec: proof of main thm}
Recall the function defined in Theorem \ref{thm: main thm}:
\begin{equation*}
    h(\bw, \bv,t) \equiv \frac{\psi(t,\bw,\bv ) - \psi_\infty(\bw,\bv )}{\ \psi_\infty(\bw,\bv )},
\end{equation*} 
which is the weighted fluctuation function around the stationary solution $\psi_\infty(\bw,\bv)$.
Then,  $h(\bw, \bv,t)$ satisfies the following partial differential equation,
\begin{equation}
\label{eq: VFP}
    \pt_th + \T h  = \L h,
\end{equation}
where 
\begin{equation*}
\begin{aligned}
    &\T = \bv \cdot\nb_\bw - \nb L(\bw)\cdot\nb_\bv\quad  \text{ is the transport operator};\\
    &\L = \frac{\gamma^{1/2}\beta}{2M}\frac1\M\nb_\bv \cd\l(\M\nb_\bv  \r)\quad\text{  is the Fokker Planck operator}.\\
 \end{aligned}
\end{equation*}
Also recall the norm $\ll \cdot \rl_*$ defined in Theorem \ref{thm: main thm}:
\begin{equation*}
\begin{aligned}
    \text{For any }h(\bw,\bv,t), g(\bw,\bv,t):\qd &\la h,g \ra_* = \int hg \M \,d\bw d\bv , \qd\ll h \rl^2_* = \int \lv h\rv^2 \M \,d\bw  d\bv ,
\end{aligned}
\end{equation*}
\begin{lemma}
\label{prop: T L}
One have the following properties for the operator $\T, \L$:
\begin{itemize}
\item[(1)] $\ds \la \T f, g \ra_\st = -\la f, \T g \ra_\st$,
\item[(2)] $\ds \la Tf, f \ra_\st = 0$,
\item[(3)] $\ds \la \L f, g \ra_\st = -\frac{\gamma^{1/2}\beta}{2M} \la \nb_\bv  f, \nb_\bv g \ra_\st$.
\end{itemize}
\end{lemma}
This lemma can be verified by direct calculations and we omit the details. These properties of operators $\L, \T$ will be frequently used later. 
\begin{lemma}
  \label{lemma: H ineq}
For the positive definite  matrix $P$ defined in (\ref{def: P}), the function $h(t,\bw ,\bv )$ satisfies
\begin{equation*}
  \begin{aligned}
    &\frac12\frac{d}{dt}H(t) +\frac12\int [ \nb_\bw h, \nb_\bv h] K [ \nb_\bw h, \nb_\bv h]^\top\M d\bw d\bv \\
    &\quad \leq  \la \nb^2\tL \nb_\bv h, \nb_\bw h \ra_\st +  \la \nb^2\tL \nb_\bv h, \nb_\bv h \ra_\st\\
  \end{aligned}
\end{equation*}
where the modified risk function  $\tL$ is defined in Assumption (A.4), and 
\begin{equation}
  \label{def of K}
  K \equiv \left[ 
    \begin{aligned} 
      & 2\h{C}I_d  &(C - \o^2+\g\h{C})I_d\\
      &(C - \o^2+\g\h{C})I_d &(2\g C - 2\o^2\h{C})I_d
    \end{aligned}\right].
\end{equation}
\end{lemma}
\begin{proof}
Taking the gradient $\nb_\bw $ to (\ref{eq: VFP}) and multiplying it by $\nb_\bw  h \M$ gives, 
\begin{equation*}
    \frac12\pt_t\ll \nb_\bw  h \rl^2_\st - \la T\nb_\bw  h , \nb_\bw  h \ra_\st - \la \nb^2L \nb_\bv h, \nb_\bw  h \ra_\st = \la \L \nb_\bw  h, \nb_\bw h \ra_\st
\end{equation*}
Them, applying Lemma \ref{prop: T L} yields,
\begin{equation*}
    \frac12\pt_t\ll \nb_\bw  h \rl^2_\st - \la \nb^2L \nb_\bv h, \nb_\bw  h \ra_\st = -\frac{\gamma^{1/2}\beta}{2M}\sum_{i = 1}^d \ll \pt_{v_i}\nb_\bw  h \rl^2_\st.
\end{equation*}
By Assumption (A.4), we have
\begin{equation}
\label{temp_1}
    \frac12\pt_t\ll \nb_\bw  h \rl^2_\st - \o^2\la \nb_\bv h, \nb_\bw  h \ra_\st = -\frac{\gamma^{1/2}\beta}{2M}\sum_{i = 1}^d \ll \pt_{v_i}\nb_\bw  h \rl^2_\st + \la \nb^2\tL \nb_\bv h, \nb_\bw  h \ra_\st.
\end{equation}
Similarly, taking the gradient $\nb_\bv $ to (\ref{eq: VFP}), multiplying it by $\nb_\bv  h \M$ and applying Lemma \ref{prop: T L} gives, 
\begin{equation}
\label{temp_2}
    \frac12\pt_t\ll \nb_\bv  h \rl^2_\st + \la \nb_\bw h, \nb_\bv  h \ra_\st = -\frac{\gamma^{1/2}\beta}{2M}\sum_{i = 1}^d \ll \pt_{v_i}\nb_\bv  h \rl^2_\st -  \frac{1 - \xi}{\sqrt{\g}}\ll \nb_\bv h \rl^2_\st.
\end{equation}
Taking the gradient $\nb_\bv $ to (\ref{eq: VFP}) and multiply it by $\nb_\bw  h \M$, then taking the gradient $\nb_\bw $ to (\ref{eq: VFP}) and multiply it by $\nb_\bv  h \M$, and combine the results gives,
\begin{equation}
\label{temp_3}
\begin{aligned}
    &\pt_t\la \nb_\bw h, \nb_\bv  h \ra_\st -\o^2 \la \nb_\bw h, \nb_\bv  h \ra_\st + \ll \nb_\bw h \rl^2_\st \\
    =& -\frac{\gamma^{1/4}\sqrt{\beta}}{\sqrt{M}}\sum_{i = 1}^d \la \pt_{v_i}\nb_\bv  h , \pt_{v_i}\nb_\bw  h \ra_\st -  \frac{1 - \xi}{\sqrt{\g}}\la \nb_\bv h, \nb_\bw h \ra_\st + \la \nb^2\tL\nb_\bv h, \nb_\bv  h \ra_\st.
\end{aligned}
\end{equation}
Finally, \eqref{temp_1} $+$ $C\cdot $\eqref{temp_2} $+$ 2$\hC\cdot$ \eqref{temp_1} yields,
\begin{equation}
\label{eqn:12partht}
\begin{aligned}
    &\frac12\pt_tH(t) + \frac{\gamma^{1/2}\beta}{2M}\sum_{i = 1}^d \int [ \pt_{v_i}\nb_\bw  h, \pt_{v_i}\nb_\bv  h]^\top P   [ \pt_{v_i}\nb_\bw  h, \pt_{v_i}\nb_\bv  h] d\bw d\bv  \\
    &+ \frac12\int [\nb_\bw h, \nb_\bv  h]^\top K [\nb_\bw h, \nb_\bv  h] d\bw d\bv   = \la \nb^2\tL\nb_\bw h, \nb_\bv  h \ra_\st + \la \nb^2\tL\nb_\bv h, \nb_\bv  h \ra_\st,
\end{aligned}
\end{equation}
where function $H(t)$ and the positive definite matrix $P$ are  defined in (\ref{def: P}). The positive definite property of $P$ implies that
\begin{equation*}
\frac{\gamma^{1/2}\beta}{2M}\sum_{i = 1}^d \int [ \pt_{v_i}\nb_\bw  h, \pt_{v_i}\nb_\bv  h]^\top P   [ \pt_{v_i}\nb_\bw  h, \pt_{v_i}\nb_\bv  h] d\bw d\bv\geq 0,
\end{equation*}
which together with (\ref{eqn:12partht}) complete the proof.
\end{proof}

\begin{lemma}\label{lemma: KP}
For $P,K$ defined in (\ref{def: P}) and (\ref{def of K}), respectively, there exists $\mu$, $C$, and $\hC$ such that
\begin{equation*}
  K \geq 2 \mu P\geq 0,
\end{equation*}
where value of $\mu$, $C$,  $\h{C}$ can be quantifies as follows:
\begin{equation*}
    \l\{ \begin{aligned}
      &\text{when } \frac{1 - \xi}{\sqrt{\g}} < 2\o: \mu \equiv  \frac{1 - \xi}{\sqrt{\g}}, \ C \equiv \o^2,\ \hC \equiv  \frac{1 - \xi}{2\sqrt{\g}};\\
      &\text{when } \frac{1 - \xi}{\sqrt{\g}} \geq 2\o: \mu \equiv  \frac{1 - \xi}{\sqrt{\g}} - \sqrt{ \frac{(1 - \xi)^2}{\g} - 4\o^2}, \ C \equiv  \frac{(1 - \xi)^2}{2\g}-\o^2, \ \hC \equiv  \frac{1 - \xi}{2\sqrt{\g}}.
    \end{aligned}
    \r.
  \end{equation*} \end{lemma}
This lemma can be verified by direct calculations and we omit the details.  We now go back to the proof of Theorem \ref{thm: main thm}.

\paragraph{Proof of Theorem \ref{thm: main thm}.}

By Lemmas \ref{lemma: H ineq}, \ref{lemma: KP}, and Assumption (A.4), we obtain
\begin{equation*}
  \begin{aligned}
    &\frac12\frac{d}{dt}H(t) + \mu H(t) \leq  \frac{1+\sqrt2}2b(\ll \nb_\bw h \rl^2_\st + \ll \nb_\bv  h \rl^2_\st)\\
  \end{aligned}
\end{equation*}
Let $\lam_{\min}$ be the smallest eigenvalue of the positive definite matrix $P$, we have
\begin{equation}
\label{eq: H H1}
  \begin{aligned}
    &\lam_{\min} (\ll \nb_\bw h \rl^2_\st + \ll \nb_\bv  h \rl^2_\st) \leq H(t), 
  \end{aligned}
\end{equation}
which implies,
\begin{equation*}
  \begin{aligned}
    &\frac12\frac{d}{dt}H(t) +( \mu - \kk) H(t) \leq 0,\\
  \end{aligned}
\end{equation*}
where $\ds \kk =  \frac{1+\sqrt2}2\frac{b}{\lam_{\min}} $. Solving the above inequality yields,
\begin{equation*}
  \begin{aligned}
    &H(t)\leq e^{-2(\mu - \kk)t} H(0).\\
  \end{aligned}
\end{equation*}
Inserting this inequality to (\ref{eq: H H1}) gives,
\begin{equation}
\label{eq: temp}
  \begin{aligned}
    &\ll \nb_\bw h \rl^2_\st + \ll \nb_\bv  h \rl^2_\st \leq \frac{1}{\lam_{\min}}e^{-2(\mu - \kk)t} H(0).\\
  \end{aligned}
\end{equation}
Besides, the Poincar\'e inequality w.r.t. the measure $\psi_\infty(\bw,\bv)$ is,
\begin{equation*}
  \begin{aligned}
    &\ll \nb_\bw h \rl^2_\st + \ll \nb_\bv  h \rl^2_\st \geq \frac{2M(1 - \xi)}{\gamma\beta}\min\{ C_P, d\} \ll h \rl^2_\st.   \end{aligned}
\end{equation*}
Inserting it back to (\ref{eq: temp}) leads to, 
\begin{equation*}
  \begin{aligned}
    &\ll h \rl^2_\st \leq \frac{\gamma\beta}{2M(1 - \xi)\min\{ C_P, d\}} \frac{1}{\lam_{\min}}e^{-2(\mu - \kk)t} H(0) \\
  \end{aligned}
\end{equation*}

\section{Networks and Dataset Used in Section \ref{sec:escapingsim}}
\label{sec:netdataexp}



The N1 network is  a \emph{shallow convolutional} network, which is  a modified AlexNet configuration (Krizhevsky et al., \cite{Krizhevsky2012imagenet}). 
Let $n\times[a,b,c,d]$ deonte a stack of $n$ convolution layers of $a$ filters and a Kernel size of $b\times c$
with stride length of $d$. Then, N1 network uses $2$ sets of $[65,5,5,2]$–MaxPool($3$) and
$2$ dense layers of sizes $(384,192)$ and finally, an output layer of size $10$. We use ReLU activations.

 The N2 network is a \emph{deep convolutional}  network, which is a modified  VGG configuration (Simonyan and Zisserman \cite{Simonyan2015}). The N2 network uses the configuration: $2\times[64, 3, 3, 1]$, $2\times [128, 3, 3, 1]$, $3\times [256, 3, 3, 1]$, $3\times [512, 3, 3, 1]$, $3\times [512, 3, 3, 1]$ and a MaxPool($2$) after each stack. This stack is followed by a  $512$-dimensional dense layer and finally, a $10$-dimensional output layer. We use ReLU activations.

The MNIST dataset (LeCun et al. \cite{LeCun1998})  contains $60,000$ training images and $10,000$ testing images, where each image is black and white and normalized to fit into a $28\times 28$ pixel bounding box and it belongs to one of total $10$ classes of handwritten digits (i.e., $0,1,2,\ldots,10$). 

The CIFAR-10 dataset consists of $50,000$ training data and $10,000$ testing data, where each data is a color image with $32\times 32$ features and it belongs to one of total $10$ classes representing airplanes, cars, birds, cats, deer, dogs, frogs, horses, ships, and trucks.


%
%

\bibliographystyle{spbasic}      


\end{document}